\documentclass[lettersize,journal]{IEEEtran}
\usepackage{amsmath,amsfonts}
\usepackage{algorithm}
\usepackage[noend]{algpseudocode}
\usepackage{array}
\usepackage[caption=false,font=normalsize,labelfont=sf,textfont=sf]{subfig}
\usepackage{textcomp}
\usepackage{stfloats}
\usepackage{url}
\usepackage{verbatim}
\usepackage{graphicx}
\usepackage{cite}

\usepackage{accents}
\usepackage{amssymb}
\usepackage{amsthm}
\usepackage{bm}
\usepackage[font=footnotesize,skip=1pt]{caption}
\usepackage{gensymb}
\usepackage{hyperref}
\usepackage{cleveref} 
\usepackage{enumitem}
\usepackage{mathtools}
\usepackage{multirow}
\usepackage{nccmath}
\usepackage{optidef}
\usepackage{scalerel}
\usepackage{siunitx}
\usepackage{svg}
\usepackage{tikz}
\usepackage{ulem}
\usepackage{xcolor}

\newcommand\publicationtext{%
  \scriptsize This article has been accepted for publication in IEEE Transactions on Robotics. It is accessible via the following DOI link: https://doi.org/10.1109/TRO.2025.3567529. This is the author's version which may differ from the final publication.}

\newcommand\publicationnotice{%
  \begin{tikzpicture}[remember picture,overlay]
    \node[anchor=north,draw=none,yshift=2pt] at (current page.north) {\parbox{\dimexpr0.8\textwidth-\fboxsep-\fboxrule\relax}{\publicationtext}};
  \end{tikzpicture}%
}

\newcommand\copyrighttext{%
  \scriptsize \textcopyright 2025 IEEE. Personal use of this material is permitted. Permission from IEEE must be obtained for all other uses, in any current or future media, including reprinting/republishing this material for advertising or promotional purposes, creating new collective works, for resale or redistribution to servers or lists, or reuse of any copyrighted component of this work in other works.}

\newcommand\copyrightnotice{%
  \begin{tikzpicture}[remember picture,overlay]
    \node[anchor=south,draw=none,yshift=-2pt] at (current page.south) {\parbox{\dimexpr0.8\textwidth-\fboxsep-\fboxrule\relax}{\copyrighttext}};
  \end{tikzpicture}%
}

\usetikzlibrary{svg.path}

\definecolor{orcidlogocol}{HTML}{A6CE39}
\tikzset{
  orcidlogo/.pic={
    \fill[orcidlogocol] svg{M256,128c0,70.7-57.3,128-128,128C57.3,256,0,198.7,0,128C0,57.3,57.3,0,128,0C198.7,0,256,57.3,256,128z};
    \fill[white] svg{M86.3,186.2H70.9V79.1h15.4v48.4V186.2z}
                 svg{M108.9,79.1h41.6c39.6,0,57,28.3,57,53.6c0,27.5-21.5,53.6-56.8,53.6h-41.8V79.1z M124.3,172.4h24.5c34.9,0,42.9-26.5,42.9-39.7c0-21.5-13.7-39.7-43.7-39.7h-23.7V172.4z}
                 svg{M88.7,56.8c0,5.5-4.5,10.1-10.1,10.1c-5.6,0-10.1-4.6-10.1-10.1c0-5.6,4.5-10.1,10.1-10.1C84.2,46.7,88.7,51.3,88.7,56.8z};
  }
}

\newcommand{\norm}[1]{\lVert#1\rVert}

\newcommand\orcidicon[1]{\href{https://orcid.org/#1}{\mbox{\scalerel*{
\begin{tikzpicture}[yscale=-1,transform shape]
\pic{orcidlogo};
\end{tikzpicture}
}{|}}}}

\newcommand{\redsout}{\bgroup\markoverwith{\textcolor{red}{\rule[0.3ex]{2pt}{1pt}}}\ULon}

\algrenewcommand\alglinenumber[1]{\footnotesize\textbf{#1}}
\algblockdefx{IndentBlock}{EndIndentBlock}{\vspace{-\baselineskip}}{\vspace{-\baselineskip}}
\algblockdefx{InnerIndentBlock}{EndInnerIndentBlock}{\vspace{-\baselineskip}}{\vspace{-\baselineskip}}
\newcommand{\StateNoIndent}{\Statex\hspace{-\algorithmicindent}}

\newtheorem{theorem}{Theorem}
\newtheorem{corollary}{Corollary}

\newtheorem{assumption}{Assumption}
\newtheorem{proposition}{Proposition}
\newtheorem{remark}{Remark}
\theoremstyle{definition}

\theoremstyle{property}
\newtheorem{property}{Property}

\title{Embedded Hierarchical MPC for Autonomous Navigation}

\author{Dennis Benders$^{\orcidicon{0000-0002-6648-7128}}$, Johannes K\"{o}hler$^{\orcidicon{0000-0002-5556-604X}}$, Thijs Niesten$^{\orcidicon{0009-0005-1658-3110}}$, Robert Babu\v{s}ka$^{\orcidicon{0000-0001-9578-8598}}$,\\ Javier Alonso-Mora$^{\orcidicon{0000-0003-0058-570X}}$, and Laura Ferranti$^{\orcidicon{0000-0003-3856-6221}}$%
\thanks{This work was supported by the National Police of the Netherlands. All content represents the opinion of the author(s), which is not necessarily shared or endorsed by their respective employers and/or sponsors. Laura Ferranti received support from the Dutch Science Foundation NWO-TTW Foundation within the Veni project HARMONIA (nr. 18165). Johannes K\"ohler was supported by the Swiss National Science Foundation under NCCR Automation (grant agreement 51NF40 180545).}%
\thanks{Dennis Benders, Thijs Niesten, Robert Babu\v{s}ka, Javier Alonso-Mora and Laura Ferranti are with the department of Cognitive Robotics, Delft University of Technology, 2628 CD Delft, The Netherlands (email: \{d.benders, t.e.j.niesten, r.babuska, j.alonsomora, l.ferranti\}@tudelft.nl).}%
\thanks{Robert Babu\v{s}ka is also with the Czech Institute of Informatics, Robotics and Cybernetics, Czech Technical University in Prague.}%
\thanks{Johannes K\"{o}hler is with the Institute for Dynamic Systems and Control, ETH Z\"{u}rich, CH-8092, Switzerland (email: jkoehle@ethz.ch).}%
}

\IEEEaftertitletext{\vspace{-0.9\baselineskip}}

\begin{document}

\maketitle

\begin{abstract}
To efficiently deploy robotic systems in society, mobile robots must move autonomously and safely through complex environments. Nonlinear model predictive control (MPC) methods provide a natural way to find a dynamically feasible trajectory through the environment without colliding with nearby obstacles. However, the limited computation power available on typical embedded robotic systems, such as quadrotors, poses a challenge to running MPC in real time, including its most expensive tasks: constraints generation and optimization. To address this problem, we propose a novel hierarchical MPC scheme that consists of a planning and a tracking layer. The planner constructs a trajectory with a long prediction horizon at a slow rate, while the tracker ensures trajectory tracking at a relatively fast rate. We prove that the proposed framework avoids collisions and is recursively feasible. Furthermore, we demonstrate its effectiveness in simulations and lab experiments with a quadrotor that needs to reach a goal position in a complex static environment. The code is efficiently implemented on the quadrotor's embedded computer to ensure real-time feasibility. Compared to a state-of-the-art single-layer MPC formulation, this allows us to increase the planning horizon by a factor of 5, which results in significantly better performance.
\end{abstract}


\section*{Supplementary material}
The open-source implementation of this work is available at https://github.com/dbenders1/hmpc.

The video accompanying this work can be found at https://youtu.be/0RnrKk6830I.

\publicationnotice
\copyrightnotice
  
\section{Introduction}\label{sec:introduction}
\IEEEPARstart{A}{utonomous} mobile robots play an increasingly important role in our society \cite{siegwart2011introduction}. The application domains are widespread, including self-driving cars \cite{paden2016survey} and environment exploration in search and rescue operations \cite{delmerico2019current}. To successfully perform its task in such applications, a robot typically needs to navigate through a complex environment and reach a goal \cite{brito2021go,lodel2022look}. This requires solving the motion planning and control problem: the robot needs to plan a smooth, collision-free, and dynamically feasible trajectory to avoid unnecessary braking and remain safe.

\begin{figure}[t]
    \centering
    \includegraphics[width=\columnwidth]{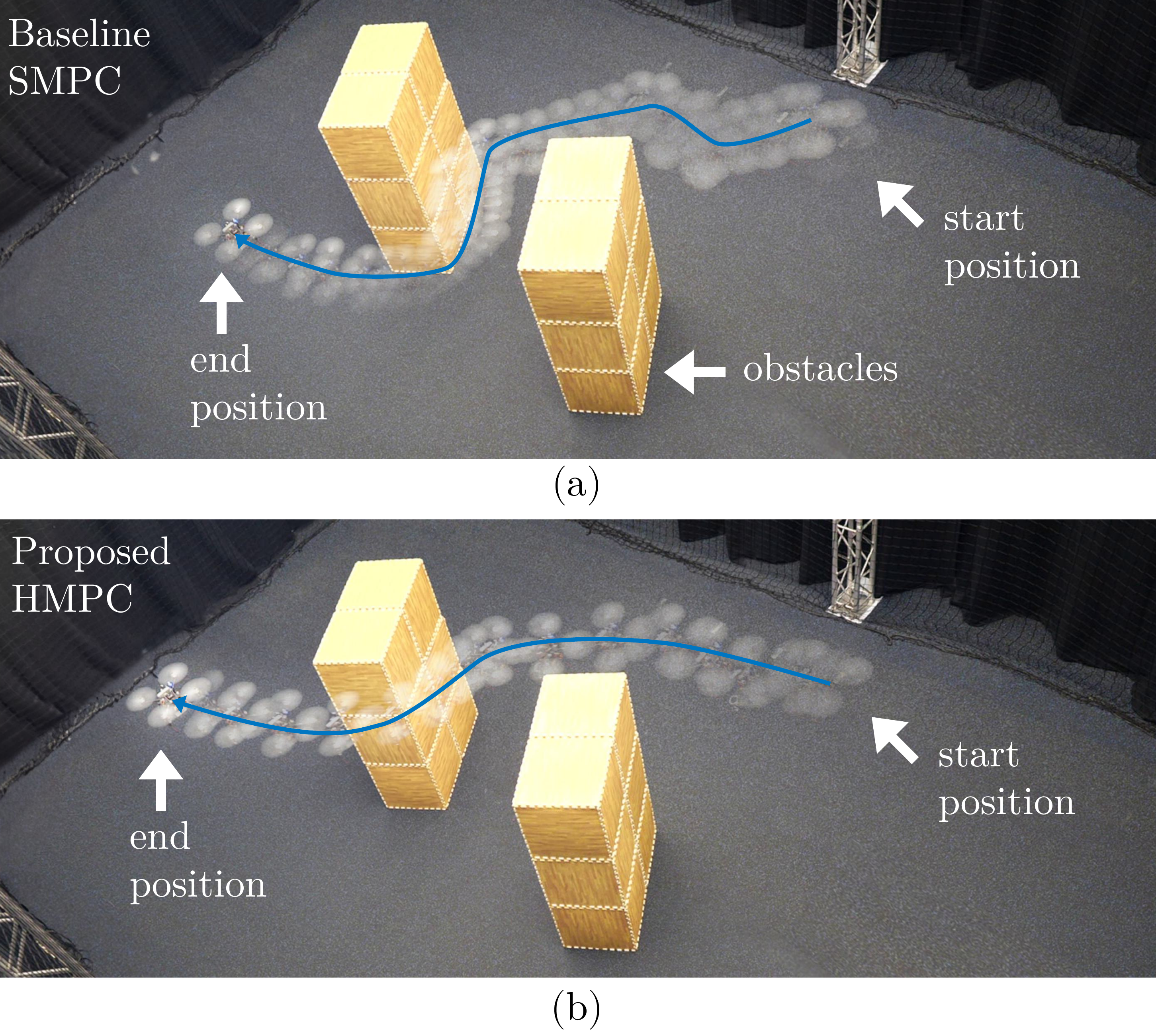}
    \caption{The closed-loop trajectory of (a) the baseline single-layer MPC (SMPC) scheme and (b) the proposed hierarchical MPC (HMPC) scheme. Whereas SMPC solves the planning and tracking tasks in a single-layer formulation, HMPC decouples the tasks using a planning MPC and a tracking MPC that use the same nonlinear model but run at different frequencies. SMPC uses the same nonlinear model as PMPC and TMPC and is sampled at the same rate as TMPC to ensure reliable operation. Both schemes fly from start to goal position and avoid the obstacles. Compared to SMPC, HMPC has a significantly longer planning horizon, given the limited computational resources on the onboard computer, and does not have to balance planning and tracking tasks. Therefore, HMPC reaches the goal faster, is less sensitive to model mismatch, and maintains altitude better than SMPC.}
    \label{fig:smpc_hmpc_front}
\end{figure}

A natural way to solve this problem is to leverage nonlinear model predictive control (MPC) \cite{rawlings2017model}. MPC is an effective method that, based on a nonlinear system model and an online observed environment, can optimize for a smooth, dynamically feasible, and collision-free trajectory. Using MPC, one can incorporate model, actuator, system, and obstacle avoidance constraints in a straightforward way, which, if constructed properly, provides safety guarantees of the closed-loop system \cite{kamel2017linear,romero2022model,santos2023nonlinear}.\looseness=-1

A common approach is to use a \underline{single-layer} MPC scheme that executes trajectory planning and tracking as a single optimization problem \cite{neunert2016fast,kamel2017linear,brito2019model,de2021scenario,saccani2022multitrajectory,romero2022model,santos2023nonlinear}. However, the main drawback of this approach is the time to solve the problem, which increases with increasing prediction horizon (for long-term predictions) and decreasing sampling times (for fast feedback), meaning that the user has to make a trade-off between long-term planning and accurate tracking. This trade-off leads to suboptimal solutions, especially on mobile robots that carry embedded hardware with limited computing power.

To overcome this problem, recent works consider decoupling the planning and tracking tasks in the form of a \underline{hierarchical} MPC (HMPC) scheme, where a planning MPC (PMPC) computes a long-term trajectory, which is tracked by a fast tracking MPC (TMPC) \cite{falcone2008hierarchical,gray2012predictive,liniger2015optimization,elsayed2023generic,ibrahim2019hierarchical,kogel2023safe}. In HMPC, the PMPC must generate a dynamically feasible trajectory that the TMPC is able to track. Although necessary for reliable operation, some of the mentioned works do not consider the co-design of PMPC and TMPC, i.e., designing the PMPC based on the tracking capabilities of the TMPC. Furthermore, to ensure runtime feasibility, these works either consider linearizing the nonlinear model or optimizing a plan from a pre-defined set of motion primitives, thereby limiting the maneuverability of the nonlinear system.

To address both problems, we propose a novel HMPC framework. In this framework, PMPC and TMPC use the same nonlinear robot model but are sampled at different rates. This allows planning a dynamically feasible trajectory far ahead and tracking it accurately; see Figure~\ref{fig:smpc_hmpc_front}. The TMPC and PMPC are co-designed so that we can formally prove collision avoidance and recursive feasibility for the overall HMPC framework in static environments. Furthermore, we show the runtime feasibility of the framework on embedded hardware.\looseness=-1

\subsection{Related work}\label{subsec:related_work}
The problem of planning and tracking a trajectory for a mobile robot has been addressed in the literature using different methods. Popular planning methods include reactive \cite{khatib1986real,fox1997dynamic}, sampling-based \cite{karaman2011sampling}, optimization-based \cite{tordesillas2021faster} or a combination of sampling and optimization-based \cite{richter2016polynomial} methods. The planner output can take various forms of plans, including paths \cite{quan2021eva}, splines \cite{zhou2019robust}, and time-parameterized kinodynamic trajectories \cite{tordesillas2021faster}. A general overview of motion planning methods can be found in \cite{paden2016survey,zhou2022review,orthey2023sampling}.

Depending on the type of plan, various tracking methods might be preferred. An overview of trajectory tracking methods on road and aerial vehicles is given in \cite{li2021state} and \cite{papachristos2018modeling}, respectively. Popular methods to design tracking controllers are proportional-integral-derivative (PID) \cite{bouabdallah2004pid}, linear-quadratic regulator (LQR) \cite{foehn2018onboard}, sum-of-squares (SOS) \cite{tobenkin2011invariant} and MPC \cite{liniger2015optimization}.

Most of the methods above focus on either the planning or the tracking task and do not consider the tracker capabilities in the planner. Consequently, they cannot guarantee the satisfaction of input and state constraints for general nonlinear robotic systems. Therefore, we consider works describing the planner and tracker co-design in the remainder of this section. We first discuss \textit{general planner-tracker schemes} that do not use MPC in both the planner and tracker. These schemes can be separated into approaches that perform planning \textit{offline} and \textit{online}. Afterward, we focus on existing \textit{HMPC} approaches.

\subsubsection{Planning offline}\label{subsubsec:planning_offline}
Offline planning is useful in static scenarios such as car racing, \cite{vazquez2020optimization} quadrotor racing \cite{foehn2021time}, and industrial ground robot applications \cite{li2021design}.

In the case of racing, there is usually no co-design of planner and tracker, and the assumption is that the planner and tracker are designed suitably and that the tracker is sampled at a sufficiently high frequency that it can accurately track the plan. Recent MPC schemes with tracking guarantees can be found in \cite{numerow2024inherently,krinner2024mpcc}. However, ensuring recursive feasibility requires restrictive terminal equality constraints.

In \cite{li2021design}, the authors show the performance of a hierarchical scheme in which a trajectory is planned and the velocity is smoothed offline, and the plan is tracked with MPC online on an industrial autonomous ground vehicle (AGV).

It should be noted that a significant number of works in the offline planning field show embedded hardware results since it is easier to obtain runtime feasibility for these algorithms compared to online planning methods. However, they can only be applied to solve a pre-defined task in a known environment.

\subsubsection{Planning online}
Online planning is relevant in scenarios where the perceived environment information of the robot changes during runtime. The main challenges of online planning methods are ensuring runtime feasibility and constructing a plan such that the closed-loop system avoids collisions despite deviating from the plan. Three relevant approaches to quantify this deviation are \textbf{Lyapunov-based}, \textbf{Hamilton-Jacobi (HJ) reachability-based}, and \textbf{contraction-based} methods.

Common \textbf{Lyapunov-based methods} design funnels around offline-designed motion primitives \cite{tedrake2010lqr,majumdar2017funnel} from which the composable ones are selected during runtime \cite{burridge1999sequential}. The disadvantage is that the robot is limited to the set of motion primitives, which can lead to conservative plans. This problem is addressed in recent works using parameterized funnels \cite{kousik2020bridging} or a combination of sampling and invariant set analysis \cite{m2023pip}. Other relevant Lyapunov-based methods include explicit reference governors (ERGs) \cite{nicotra2018explicit} and control barrier functions (CBFs) \cite{ames2016control,rosolia2022unified,csomay2022multi}.

One of the possible applications of \textbf{HJ reachability-based methods} is to derive a tracking error bound (TEB) between a low-fidelity planning and high-fidelity tracking model in order to reduce planning computation time \cite{falcone2008hierarchical,chen2021fastrack}. However, the TEB is computationally expensive and can be overly conservative. To address these problems, \cite{althoff2015online} reduces computation time by linearizing around operating points and constructs corresponding zonotopic reachable sets, and \cite{fridovich2018planning} reduces conservatism by switching between a non-conservative ‘slow’ plan and a conservative ‘fast’ plan.

\textbf{Contraction-based methods} are based on contraction theory, which concludes about the convergence of two trajectories, the planned and closed-loop trajectories, based on the evolution of two infinitesimally close trajectories \cite{lohmiller1998contraction}. The approach in \cite{singh2023robust} leverages contraction theory and SOS programming to generate a control contraction metric (CCM) controller with a corresponding funnel around the trajectory but suffers from expensive computation time, both offline and online, and a conservative plan. The method proposed in \cite{singh2020robust} combines the approaches in \cite{singh2023robust} and \cite{chen2021fastrack} to reduce offline computations compared to \cite{chen2021fastrack} and online planning time compared to \cite{singh2023robust}.

\subsubsection{HMPC}
Similar to general planner-tracker methods, HMPC schemes also suffer from long planning computation times and propose similar solutions: using a simpler model or optimizing from a pre-defined set of motion primitives.

Following a similar idea to \cite{chen2021fastrack}, the authors in \cite{falcone2008hierarchical} use a low-fidelity nonlinear PMPC model and a high-fidelity linear time-varying TMPC from \cite{falcone2007predictive}. Other works leverage a linear PMPC model to reduce computation time. For example, \cite{ibrahim2019hierarchical} shows how to efficiently use a linear mixed-integer PMPC and nonlinear TMPC to cover an area with a quadrotor in simulation.

Optimizing based on a set of motion primitives can be done via mixed-integer programming. This approach was taken in \cite{gray2012predictive}, demonstrating results in simulation and on a real passenger car driving on a slippery surface.

Other approaches combine a linear model and a set of motion primitives. One of these works is \cite{liniger2015optimization}, in which the authors implement a runtime feasible HMPC scheme on small race cars by linearizing the nonlinear dynamics around the previously optimal trajectory and selecting trajectories from a pre-defined set of motion primitives in the PMPC. Another work is \cite{elsayed2023generic}, in which simulation results are shown of a PMPC optimizing from a set of motion primitives and aggressive modes, and a TMPC being a tube-based formulation.

To avoid limited maneuverability, the authors in \cite{kogel2023safe} propose a mixed-integer PMPC formulation that explicitly considers the distance to obstacles based on the aggressiveness mode and a TMPC being a cyclic horizon MPC. Quadrotor simulation results demonstrate that the method can safely fly from start to goal in a known static environment. However, the PMPC computation time is already at its limit for real-time feasibility while only using a linear model, and it was not implemented on embedded hardware.

In conclusion, designing an HMPC that computes a plan using the nonlinear robot dynamics online and accurately tracks it in real time on a real robot with limited computation power is challenging. The methods above solve this problem by reducing the model complexity or optimizing a plan using a set of pre-defined motion primitives. In contrast, our method uses the same nonlinear model in both MPC layers and guarantees tracking of the planned trajectory, similar to the idea of decoupling PMPC and TMPC while maintaining recursive feasibility presented in \cite{kohler2024analysis}. Leveraging recently developed MPC tracking \cite{kohler2020nonlinear} that guarantees computationally efficient tracking performance and recursive feasibility, we propose an HMPC framework for a mobile robot navigating in a static environment. The framework ensures collision avoidance and recursive feasibility, thereby taking an important step toward closing the gap between real-world mobile robot deployment and theoretical MPC feasibility and tracking guarantees.

\subsection{Contributions}\label{subsec:contributions}
The main contribution of this paper is a novel HMPC framework to enable embedded motion planning and tracking on a mobile robot. In particular, we propose:
\begin{enumerate}[label=(\Alph*)]
    \item a combined design methodology for PMPC and TMPC leveraging the same nonlinear model, in which the PMPC constructs a dynamically feasible trajectory that the TMPC is guaranteed to track by tightening the constraints in accordance with the offline terminal ingredients design of the TMPC. The proposed PMPC can solve the long-horizon planning problem completely independent of the TMPC over a longer, user-chosen time interval;
    \item theoretical guarantees that both PMPC and TMPC are recursively feasible and the HMPC framework avoids static obstacles in complex environments at all times;
    \item an efficient way to implement the convex free space decomposition scheme presented in \cite{liu2017planning} for an MPC-based planner.
\end{enumerate}

While the framework can be used on any mobile robotic platform, we implemented it on a quadrotor's NVIDIA Jetson Xavier NX embedded computer. Simulations and experiments validate the theoretical guarantees and practicality of the proposed HMPC framework.

The code base, written in C++ as a Robot Operating System (ROS) package, implements different MPC formulations (i.e., single-layer MPC, TMPC, and PMPC) and solves these problems leveraging the ForcesPro NLP solver \cite{FORCESPro, FORCESNLP}. The only change in code for the different simulations and experiments is the interface, which allows for easier debugging and enhances sim-to-real transfer. The package, together with simulation and analysis code, is contained in a Docker environment, in line with the recommendation in \cite{cervera2018try}, and is available at https://github.com/dbenders1/hmpc.

\subsection{Outline}\label{subsec:outline}
The rest of the paper is organized as follows: Section~\ref{sec:problem_formulation} describes the considered robot and environment. Furthermore, it formulates the general trajectory planning and tracking problem and explains how the hierarchical framework follows from this scheme. Section~\ref{sec:tmpc} presents the TMPC design, including terminal ingredients (terminal cost and set) design with corresponding tracking properties. Section~\ref{sec:pmpc} provides the PMPC design to construct a reference trajectory and the corresponding obstacle avoidance constraints for the TMPC to ensure recursive feasibility and obstacle avoidance of the closed-loop system. The overall HMPC framework is described in Section~\ref{sec:hmpc}, including concise pseudo-algorithms, the theoretical analysis of the framework and its components with proofs given in Appendix~\ref{app:theorem_tmpc} and~\ref{app:theorem_pmpc}, and a summary of the theoretical properties that the framework provides. Section~\ref{sec:results} provides the implementation details and experimental results. Finally, Section~\ref{sec:conclusion} concludes the paper.

\subsection{Notation}\label{subsec:notation}
$\mathbb{N}_{[a,b)}$ represents the natural numbers between $a$ and $b$, including $a$ and excluding $b$. Vectors are expressed in bold, matrices in capitals. $\norm{\cdot}$ represents the 2-norm, and the quadratic vector norm for a positive definite matrix $Q$ is written as $\norm{\bm{x}}_Q^2 = \bm{x}^\top Q \bm{x}$, $Q \succ 0$. All subscripts relate to time or indices of vectors/matrices, the rest of the information is given as superscript, i.e., $\bm{x}_t^\mathrm{p}$ represents the state at time $t$ for PMPC $\mathrm{P}$. The superscripts are omitted if the meaning is clear from the context. For the MPC problem at time $t$, the predicted state-input pair is denoted by $(\bm{x}_{\tau|t},\bm{u}_{\tau|t})$, and its optimal solution by $(\bm{x}_{\tau|t}^*,\bm{u}_{\tau|t}^*)$ with predicted time $\tau \in [0,T]$ and prediction horizon $T$. Finally, $\bm{x}_{\cdot|t}$ represents all states over the prediction horizon.\looseness=-1

\section{Problem formulation}\label{sec:problem_formulation}
\subsection{Robot description}\label{subsec:problem_formulation_robot_description}
The mobile robot system is described by the following Lipschitz continuous nonlinear dynamics:
\begin{equation}\label{eq:state_update}
    \dot{\bm{x}}_t = f(\bm{x}_t,\bm{u}_t),
\end{equation}
with states $\bm{x}_t \in \mathbb{R}^{n^\mathrm{x}}$, inputs $\bm{u}_t \in \mathbb{R}^{n^\mathrm{u}}$, time $t \in \mathbb{R}$, and continuous-time dynamics $f(\bm{x}_t,\bm{u}_t)$. The system constraints are given by a compact polytopic set $\mathcal{Z}$:
\begin{align}\label{eq:sys_constraints}
    \mathcal{Z} &\coloneqq \mathcal{X} \times \mathcal{U} \notag\\
    &= \left\{(\bm{x},\bm{u}) \in \mathbb{R}^{n^\mathrm{x}+n^\mathrm{u}} | g_{j}^\mathrm{s}(\bm{x},\bm{u}) \leq 0, \ j \in \mathbb{N}_{[1,n^\mathrm{s}]}\right\},
\end{align}
where $g_{j}^\mathrm{s}(\bm{x},\bm{u}) = L_{j}^\mathrm{s} \begin{bmatrix}\bm{x}\\\bm{u}\end{bmatrix} - \bm{l}_{j}^\mathrm{s}, L_{j}^\mathrm{s} \in \mathbb{R}^{1 \times (n^\mathrm{x}+n^\mathrm{u})}, \bm{l}_{j}^\mathrm{s} \in \mathbb{R}, j \in \mathbb{N}_{[1,n^\mathrm{s}]}$.

\subsection{Obstacle avoidance constraints formulation}\label{subsec:problem_formulation_obst_constraints}
Without loss of generality, the position states of a mobile robot are given by $\bm{p} = C \bm{x} \in \mathbb{R}^{n^\mathrm{p}}$, with dimensions depending on the robot configuration space, where $C$ is a matrix selecting the corresponding states from $\bm{x}$. Position $\bm{p}$ is used to formulate obstacle avoidance constraints, ensuring that the robot region $\mathcal{R}$ does not intersect with obstacle regions $\mathcal{O}$. The obstacle avoidance constraints at time $t$ are given by the following polytopic free space set:
\begin{equation}\label{eq:obst_constraints_general}
    \mathcal{F}_{\tau|t} \coloneqq \left\{\bm{p}_{\tau|t}\ \middle|\ g_{j,\tau|t}^\mathrm{o}(\bm{p}_{\tau|t}) \leq 0,\ j \in \mathbb{N}_{[1,n^\mathrm{o}]}\right\},
\end{equation}
where $g_{j,\tau|t}^\mathrm{o}(\bm{p}_{\tau|t}) = L_{j,\tau|t}^\mathrm{o} \bm{p}_{\tau|t} - \bm{l}_{j,\tau|t}^\mathrm{o}, L_{j,\tau|t}^\mathrm{o} \in \mathbb{R}^{1 \times n^\mathrm{p}}$, $\bm{l}_{j,\tau|t}^\mathrm{o} \in \mathbb{R},$ $j \in \mathbb{N}_{[1,n^\mathrm{o}]}$. Without loss of generality, we set $\norm{L_{j,\tau|t}^\mathrm{o}} = 1, j \in \mathbb{N}_{[1,n^\mathrm{o}]}$ by scaling the constraints.

Section~\ref{subsec:pmpc_obst_constraints} provides details on constructing the obstacle avoidance constraints.

\subsection{Trajectory planning and tracking}\label{subsec:problem_formulation_planning_tracking}
The problem we aim to solve is to generate a dynamically feasible and collision-free trajectory for a mobile robot navigating in a static environment while considering the limited computation resources of the robot's embedded computer. To solve this problem, we rely on MPC. The single-layer MPC scheme to solve this problem is given by:
\begin{subequations}\label{eq:mpc}
    \begin{alignat}{2}
        \underset{\substack{\bm{x}_{\cdot|t},\bm{u}_{\cdot|t}}}{\operatorname{min}}\ \ &\mathcal{J}(\bm{x}_{\cdot|t},\bm{u}_{\cdot|t},\bm{p}^\mathrm{g}) \label{eq:mpc_obj}\\
        \operatorname{s.t.}\ &\bm{x}_{0|t} = \bm{x}_t,&& \label{eq:mpc_state_init}\\
        &\dot{\bm{x}}_{\tau|t} = f(\bm{x}_{\tau|t},\bm{u}_{\tau|t}),&&\ \tau \in [0,T] \label{eq:mpc_state_update}\\
        &(\bm{x}_{\tau|t},\bm{u}_{\tau|t}) \in \mathcal{Z},&&\ \tau \in [0,T] \label{eq:mpc_sys_constraints}\\
        &\bm{p}_{\tau|t} \in \mathcal{F}_{\tau|t},&&\ \tau \in [0,T] \label{eq:mpc_obst_constraints}\\
        &\bm{x}_{T|t} \in \mathcal{X}^\mathrm{f}, \label{eq:mpc_term_set}
    \end{alignat}
\end{subequations}
which optimizes a trajectory by minimizing cost function $\mathcal{J}(\bm{x}_{\cdot|t},\bm{u}_{\cdot|t},\bm{p}^\mathrm{g})$ \eqref{eq:mpc_obj}, while satisfying system constraints \eqref{eq:mpc_sys_constraints} and obstacle avoidance constraints \eqref{eq:mpc_obst_constraints}, in order to steer the robot with dynamics \eqref{eq:mpc_state_update} from an initial state \eqref{eq:mpc_state_init} to a terminal set \eqref{eq:mpc_term_set} in the direction of goal position $\bm{p}^\mathrm{g}$.

As discussed in the introduction, such a single-layer MPC is challenging to implement on embedded hardware. Hence, our goal is to develop PMPC and TMPC formulations that, when combined, reach goal $\bm{p}^g$ while satisfying constraints \eqref{eq:mpc_state_init}-\eqref{eq:mpc_term_set}. PMPC and TMPC are combined so that (a) the TMPC is guaranteed to track the trajectory provided by the PMPC and (b) the overall scheme avoids collisions and is recursively feasible. In this setting, the PMPC minimizes a user-chosen planning cost, while the TMPC minimizes a tracking cost related to the reference generated by the PMPC. Both MPC formulations use the actual nonlinear dynamics \eqref{eq:state_update} for the prediction.

To achieve this goal, the main TMPC design challenge is computing terminal ingredients, including terminal cost in \eqref{eq:mpc_obj} and terminal set \eqref{eq:mpc_term_set}, such that the closed-loop system tracks the reference computed by the PMPC.

The main PMPC design challenges are (a) constructing initial state constraint \eqref{eq:mpc_state_init}, such that the communicated trajectory to the TMPC is continuous and dynamically feasible, (b) tightening the system and obstacle avoidance constraints, such that the robot does not collide with obstacles given the offline computed terminal set of the TMPC, and (c) choosing a proper terminal set ensuring recursive feasibility of the PMPC.

To summarize the co-design of PMPC and TMPC, Figure~\ref{fig:hmpc_system_diagram} provides an overview of the proposed HMPC framework. The framework is detailed in Section~\ref{sec:hmpc}.

\begin{figure}[!t]
    \centering
    \includesvg[inkscapelatex=false]{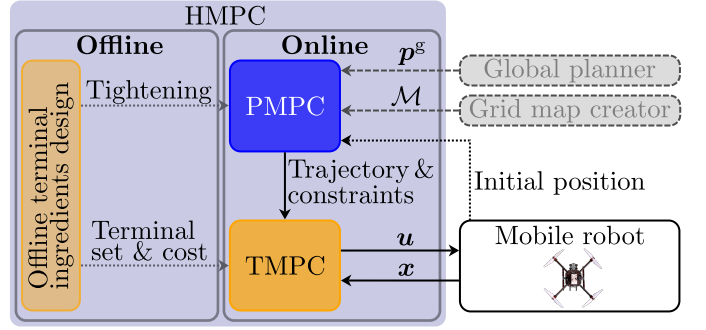}
    \caption{HMPC framework overview. In the offline phase, the terminal ingredients of the TMPC, including the terminal set, cost, and tightening constants, are computed. During online operation, the PMPC plans a trajectory based on goal position $\bm{p}^\mathrm{g}$ and grid map $\mathcal{M}$, and the TMPC tracks this trajectory using state feedback $\bm{x}$.}
    \label{fig:hmpc_system_diagram}
\end{figure}

\section{TMPC design}\label{sec:tmpc}
This section presents the TMPC formulation in Section~\ref{subsec:tmpc_formulation}. Section~\ref{subsec:tmpc_term_ing} describes the offline design of suitable terminal ingredients to prove recursive feasibility and convergence to the reference trajectory. Finally, Section~\ref{subsec:tmpc_properties} gives some properties that the reference trajectory and obstacle avoidance constraints generated by the PMPC should satisfy and formalizes the corresponding TMPC tracking properties in Proposition~\ref{prop:term_ing}.

\subsection{TMPC formulation}\label{subsec:tmpc_formulation}
The TMPC formulation is based on \cite{kohler2020nonlinear} with added obstacle avoidance capability:
\begin{subequations}\label{eq:tmpc}
    \begin{alignat}{2}
        \underset{\substack{\bm{x}_{\cdot|t},\bm{u}_{\cdot|t}}}{\operatorname{min}}\ \ &\mathrlap{\mathcal{J}^\mathrm{f,t}(\bm{x}_{T^\mathrm{t}|t}{-}\bm{x}_{T^\mathrm{t}|t}^\mathrm{r}) + \int_{\tau=0}^{T^\mathrm{t}} \mathcal{J}^\mathrm{s,t}(\bm{x}_{\tau|t},\bm{u}_{\tau|t},\bm{r}_{\tau|t})\ d\tau,}&&\hspace{200pt} \label{eq:tmpc_obj}\\
        \operatorname{s.t.}\ &\bm{x}_{0|t} = \bm{x}_t,&& \label{eq:tmpc_state_init}\\
        &\dot{\bm{x}}_{\tau|t} = f(\bm{x}_{\tau|t},\bm{u}_{\tau|t}),&&\ \tau \in [0,T^\mathrm{t}] \label{eq:tmpc_state_update}\\
        &(\bm{x}_{\tau|t},\bm{u}_{\tau|t}) \in \mathcal{Z},&&\ \tau \in [0,T^\mathrm{t}] \label{eq:tmpc_sys_constraints}\\
        &\bm{p}_{\tau|t} \in \mathcal{F}_{\tau|t},&&\ \tau \in [0,T^\mathrm{t}] \label{eq:tmpc_obst_constraints}\\
        &(\bm{x}_{T^\mathrm{t}|t}-\bm{x}_{T^\mathrm{t}|t}^\mathrm{r}) \in \mathcal{X}^\mathrm{f,t}, \label{eq:tmpc_term_set}
    \end{alignat}
\end{subequations}
with TMPC sampling time $T^\mathrm{s,t}$. This means that \eqref{eq:tmpc} is solved every $T^\mathrm{s,t}$ seconds.

The cost penalizes the error between the prediction and the reference trajectory $\bm{r} = [{\bm{x}^\mathrm{r}}^\top {\bm{u}^\mathrm{r}}^\top]^\top$ using stage and terminal costs:

\begin{align}
    \mathcal{J}^\mathrm{s,t}(\bm{x},\bm{u},\bm{r}) &\coloneqq \norm{\bm{x}-\bm{x}^\mathrm{r}}_Q^2 + \norm{\bm{u}-\bm{u}^\mathrm{r}}_R^2, \label{eq:stage_cost}\\
    \mathcal{J}^\mathrm{f,t}(\bm{x}{-}\bm{x}^\mathrm{r}) &\coloneqq \norm{\bm{x}-\bm{x}^\mathrm{r}}_P^2, \label{eq:term_cost}
\end{align}
with $Q, R, P \succ 0$. The optimal input to Problem~\eqref{eq:tmpc} is denoted by $\bm{u}_{\tau|t}^*$. As a result, the closed-loop system is given by:
\begin{equation}\label{eq:closed-loop}
    \bm{u}_{\tau+t} \coloneqq \bm{u}_{\tau|t-T^\mathrm{s,t}}^*, \quad \tau \in [0,T^\mathrm{s,t}].
\end{equation}

\subsection{Offline terminal ingredients design}\label{subsec:tmpc_term_ing}
\begin{figure*}[!t]
    \normalsize
    \begin{subequations}\label{eq:sdp}
        \begin{align}
            \underset{X,Y,{c_{j}^\mathrm{s}}^2}{\operatorname{min}}\ \ &-\operatorname{log}\operatorname{det}X + \sum_{j=1}^{n^\mathrm{s}} c_{j}^\mathrm{c} {c_{j}^\mathrm{s}}^2,\label{eq:sdp_cost}\\
            \operatorname{s.t.}\ &\begin{bmatrix}
                A(\bm{z}) X + B(\bm{z}) Y + \left(A(\bm{z}) X + B(\bm{z}) Y\right)^\top&\left(Q^\frac{1}{2} X\right)^\top&\left(R^\frac{1}{2} Y\right)^\top\\
                *&-I_{n^\mathrm{x}}&0\\
                *&0&-I_{n^\mathrm{u}}
            \end{bmatrix} \preceq 0,\ \forall \bm{z} = \begin{bmatrix}\bm{x}\\\bm{u}\end{bmatrix}\in \mathcal{Z},\label{eq:sdp_lmi_tracking}\\
            &\begin{bmatrix}
                {c_{j}^\mathrm{s}}^2&L_{j}^\mathrm{s} \begin{bmatrix}Y\\X\end{bmatrix}\\\left(L_{j}^\mathrm{s} \begin{bmatrix}Y\\X\end{bmatrix}\right)^\top&X
            \end{bmatrix} \succeq 0,\ j \in \mathbb{N}_{[1,n^\mathrm{s}]},\ X \succeq 0\label{eq:sdp_lmi_sys_constraints}.
        \end{align}
    \end{subequations}
\end{figure*}

The terminal ingredients are computed offline using the semidefinite program (SDP) \eqref{eq:sdp} with linearized system dynamics:\looseness=-1
\begin{equation}\label{eq:term_ing_A_B}
    A(\bm{z}) = \left.\frac{\partial f(\bm{x},\bm{u})}{\partial \bm{x}}\right|_{\bm{z}}, \quad B(\bm{z}) = \left.\frac{\partial f(\bm{x},\bm{u})}{\partial \bm{u}}\right|_{\bm{z}}.
\end{equation}

Based on the SDP solution, the terminal cost matrix $P$ and feedback law gain $K$ are computed using:
\begin{equation}\label{eq:P_K}
    P = X^{-1}, \quad K = Y P,
\end{equation}
and the quadratic terminal cost is an incremental Lyapunov function \cite{kohler2020nonlinear} defined in \eqref{eq:term_cost}.

The terminal set is defined as an ellipsoidal sublevel set of the incremental Lyapunov function given by:
\begin{equation}\label{eq:term_set}
    \mathcal{X}^\mathrm{f,t} \coloneqq \{\bm{x} \in \mathbb{R}^{n^\mathrm{x}}|~\norm{\bm{x}}_P^2 \leq \alpha^2\},
\end{equation}
with a user-chosen terminal set scaling $\alpha > 0$. $\mathcal{X}^\mathrm{f,t}$ is positive invariant for control input $\bm{u} = \kappa^\mathrm{f}(\bm{x},\bm{r})$ with:
\begin{equation}\label{eq:term_control_law}
    \kappa^\mathrm{f}(\bm{x},\bm{r}) \coloneqq \bm{u}^\mathrm{r} + K(\bm{x}{-}\bm{x}^\mathrm{r}),
\end{equation}
which is crucial to ensure recursive feasibility and safety.

The objective of SDP \eqref{eq:sdp} is to find a $P$ that gives a suitable terminal cost and set for trajectory tracking. The weights $c_{j}^\mathrm{c}$ are tuning parameters that are normalized with respect to the system constraint intervals to ensure equal tightening of each constraint $j \in \mathbb{N}_{[1,n^\mathrm{s}]}$ using tightening constant $c_{j}^\mathrm{s}$, which will be defined later in \eqref{eq:c_s}, and terminal set scaling $\alpha$.

Linear Matrix Inequality (LMI) \eqref{eq:sdp_lmi_tracking} ensures that the incremental Lyapunov function \eqref{eq:term_cost} exponentially decreases when applying the feedback law \eqref{eq:term_control_law} in the terminal set \eqref{eq:term_set}. LMI \eqref{eq:sdp_lmi_sys_constraints} ensures that the system constraints are satisfied for all $\bm{z} \in \mathcal{Z}$, see \cite{boyd1994linear} and \cite{kohler2020nonlinear} for more details.

Note that $\mathcal{Z}$ represents the continuous state-input space, and we optimize for $X$, $Y$, and $c_j^\mathrm{s}$ that uniformly hold for all $\bm{z} \in \mathcal{Z}$. Therefore, SDP \eqref{eq:sdp} is semi-infinite, and we need to grid or convexify $\mathcal{Z}$ to solve it. If \eqref{eq:sdp} is infeasible, one could leverage a more general formulation with state-dependent $X$ and $Y$; see \cite{kohler2020nonlinear} for details.

Note also that, since \eqref{eq:sdp} ensures exponential stability with linear feedback $K$, a feasible solution for non-holonomic systems, such as cars, only exists when enforcing a minimum velocity.

\subsection{Closed-loop properties}\label{subsec:tmpc_properties}
Property~\ref{property:ref_traj} and Property~\ref{property:obst_constraints} formalize the required properties for the reference trajectory and obstacle avoidance constraints:

\begin{property}\label{property:ref_traj}
    The reference trajectory satisfies the following construction (a) and update (b) conditions:
    \begin{enumerate}[label=(\alph*)]
        \item For $\tau \in [0,T^\mathrm{t}]$, it holds that:
        \begin{subequations}\label{eq:property_ref_traj_construction}
            \begin{align}
                \dot{\bm{x}}_{\tau|t}^\mathrm{r} &= f(\bm{x}_{\tau|t}^\mathrm{r},\bm{u}_{\tau|t}^\mathrm{r}), \label{eq:property_ref_traj_dynamics}\\
                \bm{r}_{\tau|t} &\in \bar{\mathcal{Z}}, \label{eq:property_ref_traj_sys_constraints}\\
                C \bm{x}_{\tau|t}^\mathrm{r} &\in \bar{\mathcal{F}}_{\tau|t}, \label{eq:property_ref_traj_obst_constraints}
            \end{align}
        \end{subequations}
        with
        \begin{align}\label{eq:property_ref_traj_sys_constraints_tightened}
            \bar{\mathcal{Z}} &\coloneqq \bar{\mathcal{X}} \times \bar{\mathcal{U}}\notag\\
            & = \left\{(\bm{x},\bm{u}) \in \mathbb{R}^{n^\mathrm{x} \times n^\mathrm{u}}\ \middle|\ 
            \begin{aligned}
                &g_{j}^\mathrm{s}(\bm{x},\bm{u}) + c_{j}^\mathrm{s} \alpha \leq 0,\\
                &j \in \mathbb{N}_{[1,n^\mathrm{s}]}
            \end{aligned}
            \right\},
        \end{align}
        and
        \begin{equation}\label{eq:property_ref_traj_obst_constraints_tightened}
            \bar{\mathcal{F}}_{\tau|t} \coloneqq \left\{\bm{p}_{\tau|t} \in \mathbb{R}^{n^\mathrm{p}}\ \middle|\ 
            \begin{aligned}
                &g_{j,\tau|t}^\mathrm{o}(\bm{p}_{\tau|t}) + c^\mathrm{o} \alpha \leq 0,\\
                &j \in \mathbb{N}_{[1,n^\mathrm{o}]}
            \end{aligned}
            \right\},
        \end{equation}
        with $\bar{\mathcal{Z}} \subseteq \mathcal{Z}$ and $\bar{\mathcal{F}}_{\tau|t} \subseteq \mathcal{F}_{\tau|t}$ following from tightening with $\alpha > 0$ from \eqref{eq:term_set} and:
        \begin{subequations}\label{eq:c_j}
            \begin{alignat}{2}
                c_{j}^\mathrm{s} &\coloneqq \norm{P^{-\frac{1}{2}} [I\ K^\top] L_{j}^\mathrm{s}},&&\quad j \in \mathbb{N}_{[1,n^\mathrm{s}]}, \label{eq:c_s}\\
                c^\mathrm{o} &\coloneqq \norm{P^{-\frac{1}{2}} C^\top}. \label{eq:c_o}
            \end{alignat}
        \end{subequations} \label{property:ref_traj_construction}
        \item For $\tau \in [0,T^\mathrm{s,t}]$, it holds that:
        \begin{equation}\label{eq:property_ref_traj_update}
            \bm{x}_{\tau+T^\mathrm{s,t}}^\mathrm{r} = \bm{x}_{\tau|t+T^\mathrm{s,t}}^\mathrm{r} = \bm{x}_{\tau+T^\mathrm{s,t}|t}^\mathrm{r}.
        \end{equation} \label{property:ref_traj_update}
    \end{enumerate}
\end{property}

Intuitively, Property~\ref{property:ref_traj}~\ref{property:ref_traj_construction} states that the reference trajectory is dynamically feasible \eqref{eq:property_ref_traj_dynamics}, and it satisfies tightened system constraints \eqref{eq:property_ref_traj_sys_constraints} and tightened obstacle avoidance constraints \eqref{eq:property_ref_traj_obst_constraints}. Property~\ref{property:ref_traj}~\ref{property:ref_traj_update} states that the reference trajectory is continuous, i.e. the reference trajectory in the time interval $[0, T^\mathrm{s,t}]$ of the current TMPC run is the same as the reference trajectory in the interval $[T^\mathrm{s,t}, 2T^\mathrm{s,t}]$ of the last TMPC run.

Figure~\ref{fig:alpha_levelset} illustrates an example of a 2D Lyapunov function and corresponding ellipsoidal terminal set around the reference trajectory.\looseness=-1

\begin{remark}\label{remark:alpha}
    The terminal set scaling $\alpha$ can be seen as a tuning parameter to make a trade-off between the performance of TMPC and PMPC. The smaller $\alpha$, the closer PMPC can plan to the obstacle avoidance constraints, resulting in a less conservative trajectory. The bigger $\alpha$, the bigger the terminal set for TMPC, making it easier for TMPC to find a solution. Although not formally proven in this paper, the size of the terminal region alpha is related to the inherent robustness of the TMPC, cf.~\cite{yu2014inherent}. This property ensures that closed-loop properties remain valid under sufficiently small disturbances, which is crucial for application on the real robot.
\end{remark}

\begin{figure}[t]
    \centering
    \includegraphics{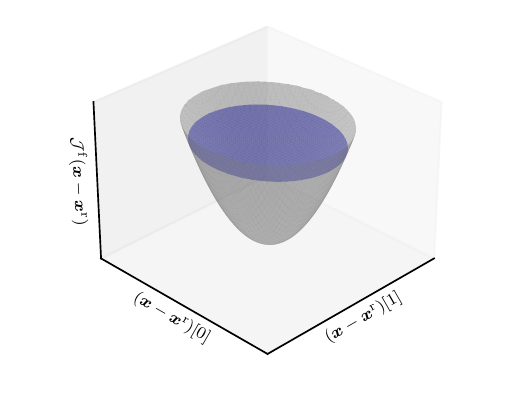}
    \caption{Quadratic Lyapunov function $\mathcal{J}^\mathrm{f,t}(\bm{x}{-}\bm{x}^\mathrm{r})$ for two state tracking error dimensions. The blue ellipsoid is the terminal set $\mathcal{X}^\mathrm{f,t}$ intersecting $\mathcal{J}^\mathrm{f,t}(\bm{x}{-}\bm{x}^\mathrm{r})$ at value $\alpha^2$.}
    \label{fig:alpha_levelset}
\end{figure}

\begin{property}\label{property:obst_constraints}
    The obstacle avoidance constraints satisfy the following construction (a) and update (b) conditions:
    \begin{enumerate}[label=(\alph*)]
        \item They form a subset such that the robot does not collide with obstacles: $\mathcal{F}_{\tau|t} \cap \mathcal{O} \oplus \mathcal{R} = \emptyset,\ \tau \in [0,T^\mathrm{p}]$. \label{property:obst_constraints_construction}
        \item They are updated in a consistent way, such that $\bm{p}_{\tau|t}^\mathrm{r} \in \bar{\mathcal{F}}_{\tau|t} \implies \bm{p}_{\tau|t}^\mathrm{r} \in \bar{\mathcal{F}}_{\tau-T^{\mathrm{s},\mathrm{p}}|t+T^{\mathrm{s},\mathrm{p}}},\ \tau \in [T^{\mathrm{s},\mathrm{p}},T^\mathrm{p}]$, with PMPC sampling time $T^{\mathrm{s},\mathrm{p}}$ and prediction horizon $T^\mathrm{p}$. \label{property:obst_constraints_update}
    \end{enumerate}
\end{property}

Intuitively, Property~\ref{property:obst_constraints}~\ref{property:obst_constraints_update} states  the reference trajectory should be
contained in the updated tightened obstacle-free region, given that it is contained in the current tightened
obstacle-free region at the same point in time.

Section~\ref{sec:pmpc} shows that Properties~\ref{property:ref_traj} and \ref{property:obst_constraints} are ensured by the PMPC design.\looseness=-1

Proposition~\ref{prop:term_ing} formalizes the properties of the TMPC terminal ingredients.\looseness=-1
\begin{proposition}[Terminal ingredients] \label{prop:term_ing}
    Suppose the terminal ingredients are designed as above, $\bm{r}$ satisfies Property~\ref{property:ref_traj} and $\mathcal{F}$ satisfies Property~\ref{property:obst_constraints}. Then for any $\bm{x}$ satisfying $(\bm{x} - \bm{x}^\mathrm{r}) \in \mathcal{X}^\mathrm{f,t}$:
    \begin{subequations}
        \begin{align}
            \frac{d}{dt}\mathcal{J}^\mathrm{f,t}(\bm{x}{-}\bm{x}^\mathrm{r}) &\leq -\mathcal{J}^\mathrm{s,t}(\bm{x},\kappa^\mathrm{f}(\bm{x},\bm{r}),\bm{r}), \label{eq:prop_term_ing_jf_decrease}\\
            (\bm{x},\kappa^\mathrm{f}(\bm{x},\bm{r})) &\in \mathcal{Z}, \label{eq:prop_term_ing_sys_constraints}\\
            C \bm{x} &\in \mathcal{F}. \label{eq:prop_term_ing_obst_constraints}
        \end{align}
    \end{subequations}
\end{proposition}
\begin{proof}
    \eqref{eq:prop_term_ing_jf_decrease} holds given a feasible solution to SDP ~\eqref{eq:sdp}. The main steps of the proof, similar to \cite{kohler2020nonlinear}, are:
    \begin{align*}
        \frac{d}{dt}\mathcal{J}^\mathrm{f,t}(\bm{\delta}) &= 2 \bm{\delta}^\top P \left(f(\bm{x},\bm{u})-f(\bm{x}^\mathrm{r},\bm{u}^\mathrm{r})\right)\\
        &\overset{\text{(a)}}{=} 2 \bm{\delta}^\top P \left(A(\bm{z}(s))+B(\bm{z}(s))K\right) \bm{\delta}\\
        &\overset{\text{(b)}}{\leq} -\bm{\delta}^\top \left(Q+K^\top R K\right) \bm{\delta}\\
        &= -\mathcal{J}^\mathrm{s,t}(\bm{x},\kappa^\mathrm{f}(\bm{x},\bm{r}),\bm{r}),
    \end{align*}
    where $\bm{\delta}{=}\bm{x}{-}\bm{x}^\mathrm{r}$. (a) uses the fact that $\bm{z}(s){=}\begin{bmatrix}\bm{x}^\mathrm{r}\\\bm{u}^\mathrm{r}\end{bmatrix}{+}s\begin{bmatrix}\bm{I}^{\mathrm{n}^x}\\K\end{bmatrix}\bm{\delta}$ linearly interpolates between $(\bm{x},\bm{u})$ and $(\bm{x}^\mathrm{r},\bm{u}^\mathrm{r})$ for $s \in [0,1]$, and the mean value theorem ensuring that this equality holds for some value of $s$. (b) follows from LMI~\eqref{eq:sdp_lmi_tracking}, see \cite[Lemma~3]{kohler2020nonlinear}.
    Furthermore, given $L \in \mathbb{R}^{1 \times n^\cdot}$ satisfying $\norm{L} = 1$ and $\bm{\delta}$ such that $\norm{\bm{\delta}}_P^2 \leq \alpha^2$ it holds that:
    \begin{align*}
        L C \bm{\delta} &\leq \norm{L C P^{-\frac{1}{2}} P^\frac{1}{2} \bm{\delta}} \leq \norm{L C P^{-\frac{1}{2}}} \norm{P^\frac{1}{2} \bm{\delta}}\\
        &\leq \norm{C P^{-\frac{1}{2}}} \norm{P^\frac{1}{2} \bm{\delta}} \leq \norm{C P^{-\frac{1}{2}}} \alpha,
    \end{align*}
    where the second and third inequalities use the triangular inequality. Applying this to $L_{j}^\mathrm{o}$ with $\norm{L_{j}^\mathrm{o}} = 1$, it holds that $L_{j}^\mathrm{o} C \bm{\delta} \leq c^\mathrm{o} \alpha$, meaning that $L_{j}^\mathrm{o} C \bm{x} \leq L_{j}^\mathrm{o} C \bm{x}^\mathrm{r} + c^\mathrm{o} \alpha \leq l_{j}^\mathrm{o},\ j \in \mathbb{N}_{[1,n^\mathrm{o}]}$ holds by \eqref{eq:property_ref_traj_obst_constraints}, given \eqref{eq:property_ref_traj_obst_constraints_tightened} and \eqref{eq:c_o}, thereby proving \eqref{eq:prop_term_ing_obst_constraints}. A similar proof holds for \eqref{eq:prop_term_ing_sys_constraints}, see \cite{kohler2020nonlinear} for more details.\looseness=-1
\end{proof}

Intuitively, Proposition~\ref{prop:term_ing} states that, given Properties~\ref{property:ref_traj} and~\ref{property:obst_constraints}, the closed-loop trajectory converges to the reference trajectory and system and obstacle avoidance constraints are satisfied when applying terminal control law \eqref{eq:term_control_law} in terminal set \eqref{eq:term_set}. Therefore, the terminal set is positively invariant under the terminal control law. This means that the shifted previously optimal solution appended with the terminal control law is a feasible, not necessarily optimal, candidate solution for the next TMPC run. Repeating this argument demonstrates that TMPC Problem~\eqref{eq:tmpc} is recursively feasible. This is formally proven in Theorem~\ref{thm:tmpc} in Section~\ref{subsec:hmpc_analysis}.

\section{PMPC design}\label{sec:pmpc}
The planner's goal is to optimize a dynamically feasible trajectory towards goal position $\bm{p}^\mathrm{g}$ such that the TMPC is guaranteed to track it. Section~\ref{subsec:pmpc_formulation} presents the PMPC formulation to optimize the trajectory. The generation of the obstacle avoidance constraints is described in Section~\ref{subsec:pmpc_obst_constraints}.

\subsection{PMPC formulation}\label{subsec:pmpc_formulation}
The PMPC formulation is given by:

\begin{subequations}\label{eq:pmpc}
    \begin{alignat}{2}
        \underset{\substack{\bm{x}_{\cdot|t},\bm{u}_{\cdot|t}}}{\operatorname{min}}&\mathrlap{\mathcal{J}^{\mathrm{f},\mathrm{p}}(\bm{x}_{T^\mathrm{p}|t},\bm{u}_{T^\mathrm{p}|t},\bm{p}^\mathrm{g})\!+\!\int_{\tau=0}^{T^\mathrm{p}} \mathcal{J}^{\mathrm{s},\mathrm{p}}(\bm{x}_{\tau|t},\bm{u}_{\tau|t},\bm{p}^\mathrm{g})\ d\tau,}&&\hspace{200pt} \label{eq:pmpc_obj}\\
        \operatorname{s.t.}\ &\bm{x}_{\tau|t} = \bm{x}_{\tau+T^\mathrm{s,p}|t-T^\mathrm{s,p}}^*, &&\ \tau \in [0,T^\mathrm{s,p}], \label{eq:pmpc_state_init}\\
        &\dot{\bm{x}}_{\tau|t} = f(\bm{x}_{\tau|t},\bm{u}_{\tau|t}),&&\ \tau \in [0,T^\mathrm{p}], \label{eq:pmpc_state_update}\\
        &(\bm{x}_{\tau|t},\bm{u}_{\tau|t}) \in \bar{\mathcal{Z}},&&\ \tau \in [0,T^\mathrm{p}], \label{eq:pmpc_sys_constraints}\\
        &\bm{p}_{\tau|t} \in \bar{\mathcal{F}}_{\tau|t},&&\ \tau \in [0,T^\mathrm{p}], \label{eq:pmpc_obst_constraints}\\
        &f(\bm{x}_{T^\mathrm{p}|t},\bm{u}_{T^\mathrm{p}|t})=0, \label{eq:pmpc_term_set}
    \end{alignat}
\end{subequations}
with PMPC sampling $T^\mathrm{s,p}$, and $\bar{\mathcal{Z}}$ and $\bar{\mathcal{F}}_{\tau|t},\ \tau \in [0,T^\mathrm{p}],$ given by \eqref{eq:property_ref_traj_sys_constraints_tightened} and \eqref{eq:property_ref_traj_obst_constraints_tightened}, respectively. As mentioned, the cost is a user-chosen planning cost that can include any terms suitable for navigating in the environment (e.g., goal-oriented MPC (GO-MPC), see \cite{brito2021go}, or MPCC, see \cite{romero2022model,brito2019model}). Section~\ref{subsec:results_impl} elaborates on the specific function used to generate the experimental results. The combination of initial state constraint \eqref{eq:pmpc_state_init} and system dynamics \eqref{eq:pmpc_state_update}, is sufficient to enforce both \eqref{eq:property_ref_traj_dynamics} in Property~\ref{property:ref_traj}~\ref{property:ref_traj_construction} and Property~\ref{property:ref_traj_update}. Furthermore, \eqref{eq:pmpc_sys_constraints} and \eqref{eq:pmpc_obst_constraints} ensure \eqref{eq:property_ref_traj_sys_constraints} and \eqref{eq:property_ref_traj_obst_constraints} in Property~\ref{property:ref_traj}~\ref{property:ref_traj_construction}, respectively. Finally, \eqref{eq:pmpc_term_set} is a terminal set constraint to ensure the recursive feasibility of the PMPC. Note that \eqref{eq:pmpc_state_init} can be efficiently implemented by using a shorter horizon $T^\mathrm{p}-T^{\mathrm{s},\mathrm{p}}$ with initial state constraint $\bm{x}_{0|t} = \bm{x}_{T^{\mathrm{s},\mathrm{p}}|t-T^{\mathrm{s},\mathrm{p}}}^*$.

\subsection{Obstacle avoidance constraints}\label{subsec:pmpc_obst_constraints}
\begin{figure*}[!ht]
    \centering
    \includesvg{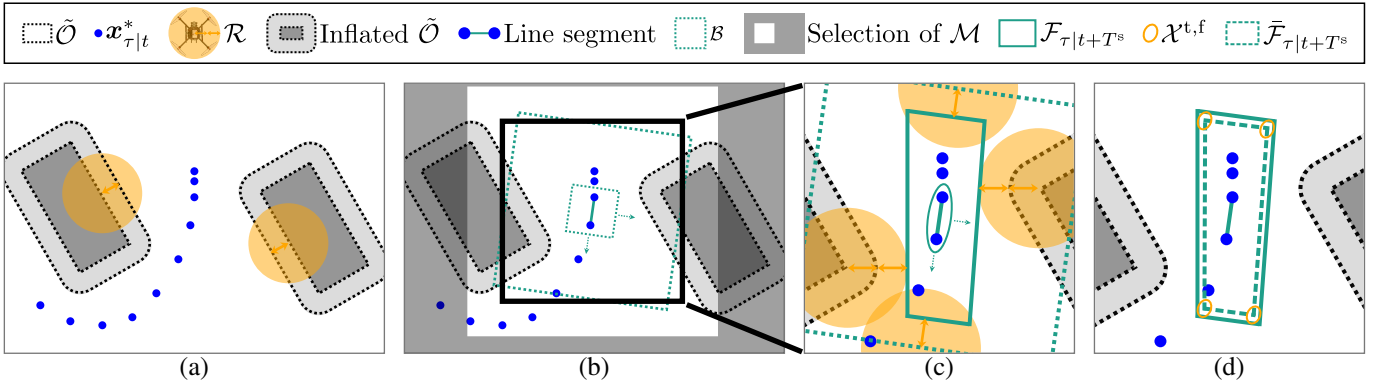}
    \caption{2D visualization of map pre-processing and \textit{I-DecompUtil}, given occupied grid cells $\tilde{\mathcal{O}}$ and the last optimized plan $\bm{x}_{\tau|t}^*$. (a) The obstacles are inflated by half of the robot radius (orange arrows). (b) To construct the obstacle avoidance constraints around a specific line segment, first, a subset of the grid map $\mathcal{M}$ is selected such that the bounding box $\mathcal{B}$ with any orientation fits in this subset. (c) The obstacle avoidance constraints $\mathcal{F}_{\tau|t+T^\mathrm{s,p}}$ are constructed according to the SFC method in \cite{liu2017planning} by growing an ellipsoid around the line segment, creating the tangential lines and clipping them to $\mathcal{B}$. Furthermore, $\mathcal{F}_{\tau|t+T^\mathrm{s,p}}$ are tightened by the other half of the robot radius, such that the robot does not collide with the obstacles if its center satisfies $\mathcal{F}_{\tau|t+T^\mathrm{s,p}}$. (d) The tightened obstacle avoidance constraints $\bar{\mathcal{F}}_{\tau|t+T^\mathrm{s,p}}$ are constructed by tightening $\mathcal{F}_{\tau|t+T^\mathrm{s,p}}$ with $c^\mathrm{o} \alpha$ according to \eqref{eq:property_ref_traj_obst_constraints_tightened}, visualized using terminal set $\mathcal{X}^\mathrm{f,t}$ in this figure. Note that (b)-(d) are repeated for all line segments.}
    \label{fig:obst_constraints}
\end{figure*}

This section describes the method to generate obstacle avoidance constraints of the form \eqref{eq:obst_constraints_general}. To deal with an online observed environment with obstacles $\mathcal{O}$, we leverage grid map representation $\mathcal{M}$ in which the occupied cells $\tilde{\mathcal{O}}$ indicate the obstacle edges. Given these occupied cells, the method should generate constraints satisfying Property~\ref{property:obst_constraints}.

In particular, the method builds on \textit{DecompUtil} as introduced in \cite[Fig. 5-8]{liu2017planning}. Based on the grid map and a linear path segment, \textit{DecompUtil} builds a convex obstacle-free region by appending the tangential line of a growing ellipsoid around the path segment at the first encountered occupied grid cell after removing all grid cells behind the previously generated tangential lines and clipping the result to bounding box $\mathcal{B}$.

In this work, the obstacle avoidance constraints are constructed following the same method, based on piecewise linear segments between optimally planned PMPC positions $\bm{p}_{iT^{\mathrm{s},\mathrm{p}}|t-T^{\mathrm{s},\mathrm{p}}}^*$ and $\bm{p}_{(i+1)T^{\mathrm{s},\mathrm{p}}|t-T^{\mathrm{s},\mathrm{p}}}^*, i \in \mathbb{N}_{[1,N-1]}, N = \frac{T^\mathrm{p}}{T^{\mathrm{s},\mathrm{p}}}$, starting with $\mathcal{B}$ around the initial position. This approach is further referred to as \textit{Iterative-DecompUtil} (\textit{I-DecompUtil}) and visualized in Figure~\ref{fig:obst_constraints}.

There are two important things to note in the figure: the \textit{Selection of $\mathcal{M}$} and the usage of the robot radius to both tighten the obstacle-free convex regions and inflate the obstacles.\looseness=-1

\textit{Selection of $\mathcal{M}$} reduces the computation time by reducing the number of grid cells that need to be considered.

\begin{figure}[!h]
    \centering
    \includesvg{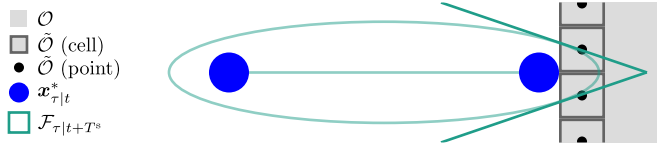}
    \caption{Visualization of a scenario in which the robot would crash into the obstacle if the constraints are not tightened. The obstacle is inflated, here represented by $\mathcal{O}$, with corresponding occupied grid cells $\tilde{\mathcal{O}}$ and the point representation of $\tilde{\mathcal{O}}$ in the code. While the ellipsoid does not intersect the grid points, it overlaps with the obstacle region, i.e., $\mathcal{F} \cap \mathcal{O} \neq \emptyset$.}
    \label{fig:obst_constraints_tightening_purpose}
\end{figure}

Since \eqref{eq:tmpc} and \eqref{eq:pmpc} optimize for the geometric robot center, the distance between obstacle avoidance constraints and occupied grid cells should be at least the robot radius. Tightening the constraints by the robot radius would result in sharp corners between subsequent obstacle-free regions around sharp obstacle corners, resulting in more conservative plans. On the other hand, inflating the obstacles by the robot radius would lead to the scenario depicted in Figure~\ref{fig:obst_constraints_tightening_purpose}, meaning that a crash may occur. To prevent both issues, half of the robot radius is used to tighten the constraints, and the other half is used to inflate the obstacles. This ensures collision avoidance, given that the robot radius is significantly larger than the grid map resolution.\looseness=-1

The resulting obstacle-free convex regions from \textit{I-DecompUtil} are piecewise-defined as:
\begin{equation}\label{eq:obst_constraints_discrete}
    \mathcal{F}_{\tau+iT^{\mathrm{s},\mathrm{p}}|t} = \mathcal{F}_{iT^{\mathrm{s},\mathrm{p}}|t},\ \tau \in (0,T^{\mathrm{s},\mathrm{p}}],\ i \in \mathbb{N}_{[0,N-1]}.
\end{equation}
This means that $\bm{p}_{(iT^\mathrm{s,p},(i+1)T^{\mathrm{s},\mathrm{p}}]|t}^* \in \mathcal{F}_{iT^\mathrm{s,p}|t}, i \in \mathbb{N}_{[0,N-1]}$. Note that, although the regions are constructed using piecewise linear segments, the nonlinear robot trajectory is contained in these convex regions since the constraints are imposed on the TMPC and PMPC trajectories through \eqref{eq:tmpc_obst_constraints} and \eqref{eq:pmpc_obst_constraints}, respectively.\looseness=-1

To prove recursive feasibility of PMPC formulation \eqref{eq:pmpc} later in Section~\ref{subsec:hmpc_analysis}, we consider the following assumption regarding the grid map $\mathcal{M}$ and the PMPC sampling time $T^{\mathrm{s},\mathrm{p}}$:
\begin{assumption}\label{ass:map}
    The grid map $\mathcal{M}$ and PMPC sampling time $T^{\mathrm{s},\mathrm{p}}$ are such that each trajectory segment $\bm{p}_{\tau+iT^{\mathrm{s},\mathrm{p}}|t-T^{\mathrm{s},\mathrm{p}}}^*, \tau \in [0,T^{\mathrm{s},\mathrm{p}}], i \in \mathbb{N}_{[0,N-1]}$ is contained within the ellipsoid to used to generate the first half-space constraint, $j = 1$ in \eqref{eq:obst_constraints_general}, of the obstacle-free region $\mathcal{F}_{iT^\mathrm{s,p}|t}, i \in \mathbb{N}_{[0,N-1]}$.
\end{assumption}
\begin{remark}
    This assumption becomes a limiting factor in scenarios with a high obstacle density, a long PMPC sampling interval, and a high velocity. If this assumption is not satisfied, the PMPC sampling time or the maximum velocity constraint needs to be reduced. This was not a limiting factor in the simulations and experiments in Section~\ref{sec:results}.
\end{remark}

Combining \eqref{eq:obst_constraints_discrete}, Assumption~\ref{ass:map}, and the fact that the convex regions are generated based on the previously optimal solution $\bm{p}_{[0,T^\mathrm{p}]|t-T^\mathrm{s,p}}^*$, $\bm{p}_{[T^{\mathrm{s},\mathrm{p}},T^\mathrm{p}]|t-T^\mathrm{s,p}}^*$ also satisfies the new constraints:
\begin{subequations}\label{eq:obst_constraints}
    \begin{align}
        \bm{p}_{\tau|t-T^\mathrm{s,p}}^* &\in \mathcal{F}_{\tau-T^{\mathrm{s},\mathrm{p}}|t},\ \tau \in [T^{\mathrm{s},\mathrm{p}},T^\mathrm{p}],\\
        \bm{p}_{T^\mathrm{p}|t-T^\mathrm{s,p}}^* &\in \mathcal{F}_{T^\mathrm{p}|t}.
    \end{align}
\end{subequations}

Thus, Property~\ref{property:obst_constraints}~\ref{property:obst_constraints_construction} is ensured by the combination of Assumption~\ref{ass:map} and the fact that the obstacle avoidance constraints are generated by growing ellipsoids starting in free space until it touches $\tilde{\mathcal{O}}$. Moreover, Property~\ref{property:obst_constraints}~\ref{property:obst_constraints_update} is satisfied by \eqref{eq:obst_constraints}.

In conclusion, by formulation \eqref{eq:pmpc}, Assumption~\ref{ass:map}, and the design of the obstacle avoidance constraints generation method, both Property~\ref{property:ref_traj} and Property~\ref{property:obst_constraints} are satisfied.

The following section provides an overview of how TMPC and PMPC are co-designed and details the theoretical properties of the overall HMPC framework.

\section{HMPC framework}\label{sec:hmpc}
This section provides the design and implementation overview of the overall HMPC framework in Section~\ref{subsec:hmpc_overview}, followed by the theoretical analysis in Section~\ref{subsec:hmpc_analysis}, and a summary of the framework design and properties in Section~\ref{subsec:hmpc_summary}.

\subsection{Design and implementation overview}\label{subsec:hmpc_overview}
Figure~\ref{fig:hmpc_system_diagram} gave an overview of the HMPC framework. Given $\bm{p}^\mathrm{g}$ from a global planner, and $\mathcal{M}$ satisfying Assumption~\ref{ass:map}, the PMPC, formulated in \eqref{eq:pmpc} with sampling time $T^{\mathrm{s},\mathrm{p}}$ satisfying Assumption~\ref{ass:map} and horizon length $T^\mathrm{p}$, generates obstacle avoidance constraints $\mathcal{F}$ and optimizes a corresponding feasible reference trajectory $(\bm{x}^\mathrm{r},\bm{u}^\mathrm{r})$ towards $\bm{p}^\mathrm{g}$. Based on the state feedback $\bm{x}$, the TMPC, formulated in \eqref{eq:tmpc} with sampling time $T^{\mathrm{s},\mathrm{t}}$ and horizon length $T^\mathrm{t}$, computes the control input $\bm{u}$ required to track $(\bm{x}^\mathrm{r},\bm{u}^\mathrm{r})$.

The main motivation for this framework is its capability to solve an expensive optimization-based planning and tracking problem on an embedded computer with limited computation power in real time. Since the optimization takes a non-negligible amount of time (as shown later in Section~\ref{sec:results}), it introduces a delay in both the reference trajectory update and the control input update. To ensure the dynamic feasibility of the optimized solution, the framework is implemented with timing properties highlighted in Figure~\ref{fig:hmpc_timing}.

Both PMPC and TMPC run at constant sampling times $T^{\mathrm{s},\mathrm{p}}$ and $T^{\mathrm{s},\mathrm{t}}$, respectively, with $T^{\mathrm{s},\mathrm{p}} = \beta T^{\mathrm{s},\mathrm{t}}, \beta \geq 2$. For convenience of visualization, the figure is constructed using $\beta = 3$. Note that the value of $\beta$ used in the implementation is provided in Section~\ref{sec:results}. Furthermore, both PMPC and TMPC optimize a trajectory that becomes valid one sampling time in the future. This means that TMPC optimizes a trajectory $(\bm{u}_{\tau|t_i}^{\mathrm{t},*},\bm{x}_{\tau|t_i}^{\mathrm{t},*})$ at $t_i-T^{\mathrm{s},\mathrm{t}}$ before it becomes valid at $t_i$, based on a reference plan $(\bm{u}_{\tau|t_i}^{\mathrm{p},*},\bm{x}_{\tau|t_i}^{\mathrm{p},*})$ that the PMPC started optimizing at $t_{i-1} - T^{\mathrm{s},\mathrm{t}} = t_i - T^{\mathrm{s},\mathrm{p}} - T^{\mathrm{s},\mathrm{t}}$. At $t_i$, $\bm{u}_{0|t_i}^{\mathrm{t},*}$ is sent to the robot before the TMPC starts optimizing the next trajectory $(\bm{u}_{\tau|t_i+T^{\mathrm{s},\mathrm{t}}}^{\mathrm{t},*},\bm{x}_{\tau|t_i+T^{\mathrm{s},\mathrm{t}}}^{\mathrm{t},*})$. Thus, the TMPC optimization can take a maximum of $T^{\mathrm{s},\mathrm{t}}$ and the PMPC a maximum of $T^{\mathrm{s},\mathrm{p}}$ in execution time.\looseness=-1

\begin{figure*}
    \centering
    \includesvg{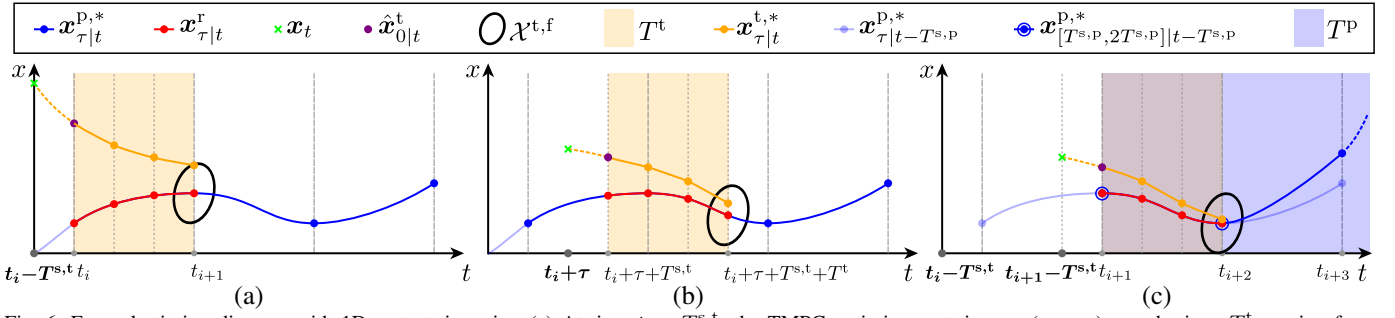}
    \caption{Example timing diagram with 1D state trajectories. (a) At time $t_i-T^{\mathrm{s},\mathrm{t}}$, the TMPC optimizes a trajectory (orange) over horizon $T^\mathrm{t}$ starting from forward-simulated state $\hat{\bm{x}}_{0|t_i}^t$ (purple), which is the model response when applying $\bm{u}_{[0,T^{\mathrm{s},\mathrm{t}}]|t_i-T^{\mathrm{s},\mathrm{t}}}^{\mathrm{t},*}$ starting from current state $\bm{x}_{t_i-T^{\mathrm{s},\mathrm{t}}}$ (green). The trajectory ends in terminal set $\mathcal{X}^\mathrm{f,t}$ around the reference trajectory $\bm{x}_{\tau|t_i}^\mathrm{r}$ (red) that becomes valid at $t_i$. $\bm{x}_{\tau|t_i}^\mathrm{r}$ is the sub-sampled version of reference plan $\bm{x}_{\tau|t_i}^{\mathrm{p},*}$ (blue). (b) In this example, the TMPC executes 2 more times ($\tau \in \{0,T^{\mathrm{s},\mathrm{t}}\}$) until the next reference trajectory becomes valid, thereby getting closer to the reference. (c) At time $t_{i+1}-T^{\mathrm{s},\mathrm{t}}$, the TMPC starts optimizing a trajectory based on a new reference plan $\bm{x}_{\tau|t_{i+1}}^{\mathrm{p},*}$ that becomes valid at $t_{i+1}$. This reference plan is optimized by the PMPC starting at time $t_i - T^{\mathrm{s},\mathrm{t}}$ and satisfies the initial state equality constraint \eqref{eq:pmpc_state_init}.}
    \label{fig:hmpc_timing}
\end{figure*}

In conclusion, the HMPC framework gives valid reference trajectories and constraints and control inputs by the satisfaction of Properties~\ref{property:ref_traj} and ~\ref{property:obst_constraints} and allows sending control commands to the robot at a fixed frequency.

\begin{remark}\label{rem:beta}
    In theory, any $\beta \in \mathbb{I}_{\geq 1}$ works provided that $T^\mathrm{p} \geq (1 + \lceil\frac{T^\mathrm{t}}{T^{\mathrm{s},\mathrm{p}}}\rceil + 2) T^\mathrm{s,p}$ such that the PMPC has at least three stages free to optimize to move towards the goal and satisfy \eqref{eq:pmpc_term_set}. If $\lceil\frac{T^\mathrm{t}}{T^{\mathrm{s},\mathrm{p}}}\rceil > 1$ one has to adjust \eqref{eq:pmpc_state_init} accordingly by constraining more PMPC stages. Given a TMPC design with fixed $T^\mathrm{s,t}$ and $T^\mathrm{t}$, increasing $\beta$ means that $T^\mathrm{s,p}$ increases, giving a coarser re-planning time and thus reduced computational load for the PMPC, and vice versa. Hence, the user can trade off PMPC execution time and performance by choosing $\beta$.
\end{remark}

Algorithm~\ref{alg:offline} and Algorithm~\ref{alg:online} give an overview of the relevant design choices of the HMPC framework.
\begin{algorithm}
    \caption{Offline design}\label{alg:offline}
    \begin{algorithmic}[1]
        \StateNoIndent \textbf{Input:} $f$, $\mathcal{Z}$
        \State $X, Y \ \ \: \leftarrow$ Solve \eqref{eq:sdp} using gridding or convexification
        \State $P, K, c_{j}^\mathrm{s}, c^\mathrm{o} \leftarrow$ Compute using \eqref{eq:P_K} and \eqref{eq:c_j}
        \State $\alpha \quad\quad\,\:\leftarrow$ Choose
        \State $\mathcal{X}^\mathrm{f,t}, \bar{\mathcal{Z}} \leftarrow$ Compute using \eqref{eq:term_set} and \eqref{eq:property_ref_traj_sys_constraints_tightened}
    \end{algorithmic}
\end{algorithm}

\begin{algorithm}
    \caption{Online implementation}\label{alg:online}
    \begin{algorithmic}[1]
        \StateNoIndent \textbf{Input:} Results from Algorithm~\ref{alg:offline}, $\mathcal{B}$, $\mathcal{R}$, $T^{\mathrm{s},\mathrm{t}}$, $T^\mathrm{t}$, $T^{\mathrm{s},\mathrm{p}}\ (\beta)$, $T^\mathrm{p}$, $\bm{p}^\mathrm{g}$ and $\mathcal{M}$
        \StateNoIndent \textbf{Define:} $t_i = i T^{\mathrm{s},\mathrm{p}}, i \in \mathbb{N}_{\geq 0}$
        \StateNoIndent \textbf{Define:} $t_{i,j} = t_i + j T^{\mathrm{s},\mathrm{t}}, j \in \mathbb{N}_{[0,\beta-1]}$
        \State $i = 0, j = 0$
        \For{$t=t_{i,j}$} \textbf{run TMPC}
            \State Apply $\bm{u}_{\tau+t}^*, \tau \in [0,T^{\mathrm{s},\mathrm{t}}]$ \eqref{eq:closed-loop}
                \State $\bm{x}_{t} \quad\ \ \ \ \leftarrow$ Measure
                \State $\bm{x}_{0|t_{i,j+1}} \leftarrow$ Forward-predict using $\bm{x}_{t}$
                \State $\mathcal{F}_{\cdot|t} \quad\ \ \,\leftarrow$ Select based on $\mathcal{F}_{\cdot|t_i}$
                \State $\bm{u}_{\cdot|t}^* \quad\ \ \,\leftarrow$ Solve \eqref{eq:tmpc} with $\bm{x}_{0|t} = \bm{x}_{0|t_{i,j+1}}$
                \State $j \qquad\ \ \:\leftarrow j + 1$
                \If{$j=\beta-1$ \textbf{and} $\bm{p}^\mathrm{g}$ not reached} \textbf{run PMPC}
                    \State $\mathcal{F}_{\cdot|t_i}, \bar{\mathcal{F}}_{\cdot|t_i} \ \ \leftarrow$ Construct as in Section~\ref{subsec:pmpc_obst_constraints}
                    \State $(\bm{x}_{\cdot|t_i}^{\mathrm{p},*}, \bm{u}_{\cdot|t_i}^{\mathrm{p},*}) \leftarrow$ Solve \eqref{eq:pmpc}
                    \State Send $(\bm{x}_{\cdot|t_i}^{\mathrm{r}}{=}\bm{x}_{\cdot|t_i}^{\mathrm{p},*}, \mathcal{F}_{\cdot|t_i})$ to TMPC
                    \State $i \qquad\qquad\ \,\leftarrow i + 1$
                    \State $j \qquad\qquad\ \leftarrow$ 0
                \EndIf
        \EndFor
    \end{algorithmic}
\end{algorithm}

\subsection{Theoretical analysis}\label{subsec:hmpc_analysis}
First, Theorem~\ref{thm:tmpc} formalizes the TMPC obstacle avoidance and recursive feasibility guarantees. Secondly, Theorem~\ref{thm:pmpc} formalizes the recursive feasibility of the PMPC. By combining these results, Corollary~\ref{cor:hmpc} concludes the recursive feasibility and obstacle avoidance of the HMPC framework.

\begin{theorem}\label{thm:tmpc}
    Suppose the terminal ingredients are designed according to Section~\ref{subsec:tmpc_term_ing}, $\bm{r}$ satisfies Property~\ref{property:ref_traj}, $\mathcal{F}$ satisfies Property~\ref{property:obst_constraints}, and \eqref{eq:tmpc} is feasible at $t=0$. Then, the resulting closed-loop system ensures that \eqref{eq:tmpc} is recursively feasible, satisfies system constraints $(\bm{x}_t,\bm{u}_t) \in \mathcal{Z}$ and avoids obstacle collisions $C \bm{x}_t \notin \mathcal{O}$ for all $t \geq 0$. Moreover, the tracking error $\norm{\bm{x}_t - \bm{x}_t^\mathrm{r}}$ asymptotically converges to zero.
\end{theorem}

The proof is detailed in Appendix~\ref{app:theorem_tmpc}. It uses a candidate solution that shifts the previously optimal solution by $T^{\mathrm{s},\mathrm{t}}$ and appends the terminal control law \eqref{eq:term_control_law}. Feasibility is ensured by the invariant terminal set, cf. Proposition~\ref{prop:term_ing}, and convergence follows by showing that the optimal cost decreases.

\begin{theorem}\label{thm:pmpc}
    Suppose that Assumption~\ref{ass:map} holds and \eqref{eq:pmpc} is initialized with a feasible steady-state plan, i.e. $\dot{\bm{x}}^\mathrm{p} = 0,\ (\bm{x},\bm{u}) \in \bar{\mathcal{Z}},\ \bm{p} \in \bar{\mathcal{F}}_{\cdot|0}$. Then, \eqref{eq:pmpc} is recursively feasible for $t \geq 0$, the optimal solution satisfies Property~\ref{property:ref_traj}, and the constructed obstacle avoidance constraints satisfy Property~\ref{property:obst_constraints}.
\end{theorem}

The recursive feasibility proof is detailed in Appendix~\ref{app:theorem_pmpc}. Similar to the proof for Theorem~\ref{thm:tmpc}, it is based on a candidate solution equal to the previously optimal solution, appended with an input ensuring steady state. Given that the previously optimal solution ends in steady state, the candidate also ends in steady state, meaning that system and obstacle avoidance constraints are always satisfied. Satisfaction of Properties~\ref{property:ref_traj} and~\ref{property:obst_constraints} follows from the reasoning explained in Section~\ref{subsec:pmpc_obst_constraints}.

\begin{corollary}\label{cor:hmpc}
    Suppose the terminal ingredients are designed according to Section~\ref{subsec:tmpc_term_ing}, Algorithm~\ref{alg:online} is initialized with a steady-state reference plan satisfying Property~\ref{property:ref_traj}, and \eqref{eq:tmpc} and \eqref{eq:pmpc} are both feasible at $t=0$. Then, \eqref{eq:tmpc} and \eqref{eq:pmpc} are recursively feasible for the closed-loop system in Algorithm~\ref{alg:online} and the closed-loop system satisfies system constraints $(\bm{x}_t,\bm{u}_t) \in \mathcal{Z}$ and avoids obstacle collisions $\mathcal{R} \cap \mathcal{O} = \emptyset$ for $t \geq 0$. Moreover, the tracking error $\norm{\bm{x}_t - \bm{x}_t^\mathrm{r}}$ asymptotically converges to zero.
\end{corollary}

By Theorem~\ref{thm:pmpc}, \eqref{eq:pmpc} is recursively feasible and ensures that the reference trajectory satisfies Property~\ref{property:ref_traj} and obstacle avoidance constraints satisfy Property~\ref{property:obst_constraints}. Then, given the initial feasibility of \eqref{eq:tmpc}, applying Theorem~\ref{thm:tmpc} gives system and obstacle avoidance constraints satisfaction and asymptotic convergence to the reference trajectory of the closed-loop system.\looseness=-1

\subsection{Summary}\label{subsec:hmpc_summary}
We solve SDP~\eqref{eq:sdp} offline for known nonlinear system dynamics \eqref{eq:state_update}. A feasible solution to \eqref{eq:sdp} gives matrices $X$ and $Y$, and tightening constants $c_j^\mathrm{s}$. $X$ and $Y$ are used to compute matrices $P$ and $K$, which are subsequently used to calculate the tightening constants $c^\mathrm{o}$ and to construct the TMPC terminal cost, set, and control law.

Note that LMI \eqref{eq:sdp_lmi_sys_constraints} ensures that $c_j^\mathrm{s}$ are Lipschitz constants, so the terminal set scaling $\alpha$ does not have to known before solving \eqref{eq:sdp} and can instead be computed as $\alpha = \frac{d}{c^\mathrm{o}}$ with $d$ the minimum distance between obstacles and reference trajectory.

Online, we compute the reference trajectory by solving PMPC Problem~\eqref{eq:pmpc} using tightening constants $c_j^\mathrm{s}$ and $c^\mathrm{o}$ and obstacle avoidance constraints generated as described in Section~\ref{subsec:pmpc_obst_constraints}. Theorem~\ref{thm:pmpc} states that the reference trajectory and corresponding constraints satisfy Properties~\ref{property:ref_traj} and~\ref{property:obst_constraints}, respectively, and that the PMPC is recursively feasible by including PMPC terminal constraint \eqref{eq:pmpc_term_set}.

Given the offline-computed terminal ingredients as proposed in Proposition~\ref{prop:term_ing} and online-computed reference trajectory satisfying Property~\ref{property:ref_traj} with corresponding obstacle avoidance constraints satisfying Property~\ref{property:obst_constraints}, Theorem~\ref{thm:tmpc} states that the online execution of TMPC Problem~\ref{eq:tmpc} guarantees asymptotic convergence to the reference trajectory, obstacle avoidance, and recursive feasibility.

Consequently, by the proposed co-design of TMPC and PMPC, Corollary~\ref {cor:hmpc} concludes that the HMPC framework is recursively feasible, satisfies system constraints, and avoids obstacle collisions at all times. Additionally, the timing of PMPC and TMPC as described in Section~\ref{subsec:hmpc_overview} allows sending control commands to the robot at a fixed frequency while maintaining the properties described above.

Important to note is that the optimized trajectory of the PMPC evolves based on the previously optimal PMPC solution, so there is no state feedback from the real system. In contrast, the state feedback is incorporated in the TMPC formulation as an initial state constraint \eqref{eq:tmpc_state_init}.

\section{Results}\label{sec:results}
This section demonstrates an efficient way to generate obstacle avoidance constraints. Furthermore, it shows the performance and advantages of the proposed HMPC framework compared to SMPC, which is commonly used in local motion planning schemes \cite{neunert2016fast,saccani2022multitrajectory,de2021scenario}. To make a fair comparison, SMPC inherits the cost function and system dynamics from PMPC, and initial state, system, and obstacle avoidance constraints from TMPC. This means, it solves \eqref{eq:pmpc}, where \eqref{eq:pmpc_state_init}, \eqref{eq:pmpc_sys_constraints} and \eqref{eq:pmpc_obst_constraints} are replaced by \eqref{eq:tmpc_state_init}, \eqref{eq:tmpc_sys_constraints} and \eqref{eq:tmpc_obst_constraints}, respectively.\looseness=-1

The results include simple simulations without model mismatch, Gazebo simulations with model mismatch, and lab experiments. All results will focus on the lab experiments and highlight the differences with simulations when relevant. The simulations and lab experiments use the same PMPC, TMPC, and SMPC settings, which are chosen to ensure real-time performance on the hardware of the real quadrotor. These settings include sampling times, horizon lengths, and weights.

We first detail the different MPC implementations and test setup in Section~\ref{subsec:results_impl}. After that, we provide the results on both obstacle avoidance constraints generation and comparison between HMPC and SMPC, in Section~\ref{subsec:results_results}.

The open-source implementation to reproduce these results is available at https://github.com/dbenders1/hmpc.

\subsection{Implementation and setup details}\label{subsec:results_impl}
The implementation and setup consider several aspects, including the considered grid map to the quadrotor model, the MPC details, the software setup, and the hardware setup, as described next.

\subsubsection{Grid map}\label{subsubsec:results_impl_map}
The considered grid map $\mathcal{M}$ is 2D and generated offline with size 12x12 m and resolution 0.01 m. The rectangular obstacles are described by their center, width, and length. The obstacle contours, formed by thin lines of occupied grid cells, are efficiently created using OpenCV\footnote{https://github.com/opencv/opencv}. As described in Section~\ref{subsec:pmpc_obst_constraints}, an extra inflation layer is designed around the obstacles to allow smoother planning of the PMPC around the obstacles. The same holds for the lines representing the environmental boundaries. $\mathcal{M}$ is the same for the simulation and lab results below for consistency.

\subsubsection{Quadrotor model}\label{subsubsec:results_impl_model}
The following quadrotor model is used in the MPC schemes:
\begin{align}
    \begin{bmatrix}
        \dot{p}^x\\\dot{p}^y\\\dot{p}^z
    \end{bmatrix} &=
    \begin{bmatrix}
        v^x\\v^y\\v^z
    \end{bmatrix}, \notag\\
    \begin{bmatrix}\dot{v}^x\\\dot{v}^y\\\dot{v}^z\end{bmatrix} &= \begin{bmatrix}s^\phi s^\psi + c^\phi s^\theta c^\psi\\-s^\phi c^\psi + c^\phi s^\theta s^\psi\\c^\phi c^\theta\end{bmatrix} a - \begin{bmatrix}0\\0\\g\end{bmatrix}, \label{eq:f}\\
    \begin{bmatrix}\dot{\phi}\\\dot{\theta}\\\dot{\psi}\\\dot{a}\end{bmatrix} &= \begin{bmatrix}-\frac{1}{\tau^\phi}&0&0&0\\0&-\frac{1}{\tau^\theta}&0&0\\0&0&-\frac{1}{\tau^\psi}&0\\0&0&0&-\frac{1}{\tau^a}\end{bmatrix} \begin{bmatrix}\phi\\\theta\\\psi\\a\end{bmatrix} + \begin{bmatrix}\frac{k^\phi}{\tau^\phi}\\\frac{k^\theta}{\tau^\theta}\\\frac{k^\psi}{\tau^\psi}\\\frac{k^a}{\tau^a}\end{bmatrix} \begin{bmatrix}\phi^\mathrm{c}\\\theta^\mathrm{c}\\\psi^\mathrm{c}\\a^\mathrm{c}\end{bmatrix}, \notag
\end{align}
with states including 3D positions, velocities and attitudes, and acceleration given by $\bm{x} = [p^x\ p^y\ p^z\ v^x\ v^y\ v^z\ \phi\ \theta\ \psi\ a]^\top$ and inputs including 3D attitude commands and collective mass-normalized thrust, or acceleration command, given by $\bm{u} = [\phi^\mathrm{c}\ \theta^\mathrm{c}\ \psi^\mathrm{c}\ a^\mathrm{c}]^\top$, sine ($s$) and cosine ($c$) expressions, gravitational constant $g$, and time and gain constants $\tau^\phi = \tau^\theta = 0.18, \tau^\psi = 0.56, \tau^a = 0.050, k^\phi = k^\theta = k^\psi = k^a = 1$ for roll, pitch, yaw and acceleration respectively.

The system constraints are given by:
\begin{align}
    -15\ \mathrm{m} &\leq p^x,\ p^y &&\leq 15\ \mathrm{m},\notag\\
    0\ \mathrm{m} &\leq p^z &&\leq 4\ \mathrm{m},\notag\\
    -2\ \mathrm{m/s} &\leq v^x,\ v^y,\ v^z &&\leq 2\ \mathrm{m/s},\notag\\
    -30\degree &\leq \phi,\ \theta,\ \psi,\ \phi^\mathrm{c},\ \theta^\mathrm{c},\ \psi^\mathrm{c} &&\leq 30\degree,\notag\\
    5\ \mathrm{m/s^2} &\leq a,\ a^\mathrm{c} &&\leq 15\ \mathrm{m/s^2}.
\end{align}

The PMPC and SMPC leverage an extended model in which the inputs are given by $[\dot{\phi}^\mathrm{c}\ \dot{\theta}^\mathrm{c}\ \dot{\psi}^\mathrm{c}\ a^\mathrm{c}]$, with $-60\degree/s \leq \dot{\phi}^\mathrm{c},\ \dot{\theta}^\mathrm{c},\ \dot{\psi}^\mathrm{c} \leq 60\degree/s$, to allow for penalizing angle rates in the cost function and thus creating a smoother trajectory. The absolute angle commands $[\phi^\mathrm{c}, \theta^\mathrm{c}, \psi^\mathrm{c}]$ are included in the states in that case.\looseness=-1

Both PMPC and TMPC implement the discretization of the quadrotor model \eqref{eq:f} using the fourth-order Runge-Kutta (RK4) integration method with a 50 ms time interval. The TMPC has a sampling time of 50 ms, requiring 1 RK4 step, whereas the PMPC has a sampling time of 500 ms, requiring $\beta = 10$ RK4 steps. Thus, although the sampling times differ between PMPC and TMPC, the model discretization is the same.

\subsubsection{TMPC offline computations}\label{subsubsec:results_impl_offline}
The offline computations described in Section~\ref{sec:tmpc} are implemented using Yalmip. The attitude states occur nonlinearly, and the acceleration state occurs linearly in the linearized system dynamics \eqref{eq:term_ing_A_B}. Therefore, the LMIs are generated using a grid of five points equally divided over the attitude state constraint intervals and the two endpoints of the acceleration constraint interval. The solution to SDP \eqref{eq:sdp} is checked for 21 points per attitude state and two points for the acceleration state constraint interval. All LMIs are satisfied at the checked points in the state space, meaning that the theoretical guarantees hold for those points. The total time to generate the LMIs, solve the SDP, and check the results is 38 s.\looseness=-1

\subsubsection{PMPC and SMPC cost function}\label{subsubsec:results_impl_pmpc_cost}
The PMPC and SMPC aim to reach the goal with a cost similar to GO-MPC \cite{brito2021go}. On a quadrotor, one can control the yaw angle independently from the position, meaning that the goal is described by a 3D position with yaw angle $(\bm{p}^\mathrm{g},\psi^\mathrm{g}) \in \mathbb{R}^4$. Consequently, the cost function is given by:
\begin{align}
    &\mathcal{J}^{\mathrm{s},\mathrm{p}}(\bm{x}_{\tau|t},\bm{u}_{\tau|t},\bm{p}^\mathrm{g},\psi^\mathrm{g})\notag\\
    &\coloneqq w^{\mathrm{xy,g}} \mathcal{H}(\bm{p}_{\tau|t}^{xy}, \bm{p}^{xy,\mathrm{g}}) + w^{\mathrm{z,g}} (p_{\tau|t}^z{-}p^{z,\mathrm{g}})^2 \notag\\
    &+ w^{\psi,\mathrm{c}} (\psi_{\tau|t}{-}\psi^\mathrm{g})^2 + w^a (a_{\tau|t}^2-g)^2 \notag\\
    &+ \norm{\bm{u}_{\tau|t}}_U^2 + w^{\phi\theta,\mathrm{c}} ({\phi_{\tau|t}^\mathrm{c}}^2 + {\theta_{\tau|t}^\mathrm{c}}^2) + w^{\psi,\mathrm{c}} {\psi_{\tau|t}^\mathrm{c}}^2,
\end{align}
with goal position weights $w^{\mathrm{xy,g}}$ and $w^{\mathrm{z,g}}$, goal yaw weight $w^{\psi,\mathrm{g}}$, acceleration weight $w^a$, input weights matrix $U = \operatorname{diag}(u_1, \dots, u_{n^\mathrm{u}})$ and, and input attitude memory state weights $w^{\phi\theta,\mathrm{c}}$ and $w^{\psi,\mathrm{c}}$. Terminal cost $\mathcal{J}^{\mathrm{f},\mathrm{p}}$ has the same expression as stage cost $\mathcal{J}^{\mathrm{s},\mathrm{p}}$, but with different weight values. $\mathcal{H}(\bm{p}_{\tau|t}^{xy}, \bm{p}^{xy,\mathrm{g}})$ denotes the Huber loss \cite{huber1992robust} of the 2D Euclidean distance from position $\bm{p}_{\tau|t}^{xy}$ towards goal position $\bm{p}^{xy,\mathrm{g}}$. The linear relation further away from the origin makes the Huber loss useful to ensure moving towards the goal with approximately constant velocity without accelerating quickly and reaching the goal slowly.

\subsubsection{PMPC and SMPC terminal steady state}\label{subsubsec:results_impl_ss}
The terminal state used to implement \eqref{eq:pmpc_term_set} in both PMPC and SMPC is given by $\{v^x{=}0, v^y{=}0, v^z{=}0, \phi{=}0, \theta{=}0, a{=}g\}$.

\subsubsection{TMPC terminal set}\label{subsubsec:results_impl_alpha}
As Section~\ref{sec:tmpc} shows, the terminal set scaling $\alpha$ is a tuning parameter to trade off PMPC and TMPC performance. In this case, $\alpha$ is computed using $\alpha = \frac{d}{c^\mathrm{o}}$, in which $d = 0.1$ m is the minimum allowed distance from obstacles to reference trajectory and $c^\mathrm{o}$ is given by \eqref{eq:c_o}.

\subsubsection{SMPC slack}\label{subsubsec:results_impl_slack}
Both the Gazebo simulations and lab experiments in Section~\ref{subsubsec:results_results_hmpc} include model mismatch. HMPC indirectly accounts for this by tightening the obstacle avoidance constraints in the PMPC. However, SMPC would become infeasible because the system would exceed the obstacle avoidance constraints. Therefore, the obstacle avoidance constraints in SMPC are implemented using a slack variable, e.g., see \cite{boyd2004convex}, allowing the system to exceed the constraints but with high penalization to steer the system back into the constraints set. Similar to the tightening of constraints in HMPC, SMPC uses a safety distance of 0.1 m to avoid collisions.

\subsubsection{PMPC, TMPC and SMPC settings}\label{subsubsec:results_impl_weights}
Table~\ref{tab:settings} summarizes the PMPC, TMPC, and SMPC settings, including the sampling times, horizon lengths, and weights. The TMPC and SMPC sampling times are chosen sufficiently low to ensure accurate control but long enough to solve the MPC optimization problems in real time. Based on the TMPC sampling time $T^{\mathrm{s},\mathrm{t}}$, the PMPC sampling time is chosen using $\beta = 10$. Given $T^{\mathrm{s},\mathrm{t}}$, this factor is selected sufficiently low for the PMPC to optimize a reasonable plan and sufficiently high for the TMPC to satisfy the terminal set constraint.

\begin{table}[!h]
    \caption{PMPC, TMPC, and SMPC settings in simulations and lab experiments.}
    \centering
    \begin{tabular}{|l|l|l|l|}
        \hline
        \textbf{Method} & $\bm{T}^{\mathrm{s}}$\textbf{ (s)} & $\bm{T}$\textbf{ (s)} & \textbf{Weighting matrices}\\\hline
        PMPC & 0.5 & 2.5 & $w^{\mathrm{xy,g^a}}{=}40, w^{\mathrm{z,g^a}}{=}40, w^{\psi,\mathrm{g^a}}{=}40,$\\
        &&&$w^\mathrm{xy,g^t}{=}200, w^\mathrm{z,g^t}{=}200, w^{\psi,\mathrm{g^t}}{=}200,$\\
        &&&$w^a = 40, w^{\phi\theta,\mathrm{c}} = 16, w^{\psi,\mathrm{c}} = 16,$\\
        &&&$U = \operatorname{diag}(16,16,16,16)$\\\hline
        TMPC & 0.05 & 0.5 & $Q{=}\operatorname{diag}(2e3,2e3,2e3,20,20,20,$\\
        &&&$\hphantom{Q{=}\operatorname{diag}(}100,100,100,100)$,\\
        &&&$R{=}\operatorname{diag}(2e3,2e3,2e3,100)$\\\hline
        SMPC & 0.05 & 0.4 & $w^{\mathrm{xy,g^a}}{=}200, w^{\mathrm{z,g^a}}{=}200, w^{\psi,\mathrm{g^a}}{=}200,$\\
        &&&$w^\mathrm{xy,g^t}{=}2e3, w^\mathrm{z,g^t}{=}2e3, w^{\psi,\mathrm{g^t}}{=}2e3,$\\
        &&&$w^a = 200, w^{\phi\theta,\mathrm{c}} = 160, w^{\psi,\mathrm{c}} = 160,$\\
        &&&$U = \operatorname{diag}(160,160,160,160)$\\\hline
    \end{tabular}
    \label{tab:settings}
\end{table}

The PMPC horizon is chosen long enough to be able to find a reasonable plan through the environments described below. The TMPC horizon is chosen to equal the PMPC sampling time. Furthermore, the SMPC horizon is reduced until the MPC scheme is runtime feasible on the hardware described below.

The GO objective weights used in PMPC and SMPC are chosen such that the terminal stage and other stages contribute equally to the cost function. Note that PMPC and SMPC have a similar weight ratio between the terminal stage and the rest of the horizon for a fair comparison. One of the differences is that the PMPC GO objective stage weights are increased to gain aggressiveness and improve goal-reaching performance, which is not possible for SMPC for stability reasons. Stability is ensured in both PMPC and SMPC by increasing the input weights. Note that this increase is more significant for SMPC than for PMPC.

The TMPC position state and attitude input weights are chosen large enough to ensure that the reference positions are accurately tracked to avoid collisions, and the system will not aggressively compensate for any reference tracking error, thereby increasing model mismatch. To further enhance stable flight, the acceleration input and state weights are increased. Since the reference acceleration is given in the body frame, its value is only valid for attitude values resulting in the corresponding body frame. Therefore, the attitude weights are increased as well.

Since we compute the terminal cost offline, we hard-code the TMPC weighting matrices to save computation time during runtime. In contrast, the PMPC and SMPC weights are not hard-coded to allow the user to change them without generating a new solver.

Note that the costs are given in continuous time and discretized using Euler.

\subsubsection{Software setup}\label{subsubsec:results_impl_software}
In all simulations and lab experiments, the ForcesPro NLP solver \cite{FORCESPro, FORCESNLP} is leveraged, using x86 compilation for simulations and ARM compilation for lab experiments. Furthermore, in both Gazebo simulations and lab experiments, PX4 stable release v1.12.3 is leveraged, using Posix compilation for simulations and NuttX compilation for lab experiments. The advantage of this setup is that both the solver and low-level controllers are the same in simulations and lab experiments to reduce sim-to-real transfer.

\subsubsection{Hardware setup}\label{subsubsec:results_impl_hardware}
The quadrotor used to generate the lab experiment results is the HoverGames drone\footnote{\url{https://nxp.gitbook.io/hovergames/}}, with the flight controller being replaced by a Pixhawk 6X mini and an added NVIDIA Jetson Xavier NX embedded computer, which has a 6-core 1400MHz NVIDIA Carmel ARMv8.2 processor. The position and orientation of the quadrotor are tracked using the Vicon Vantage V5 camera system and sent via Wi-Fi to the embedded computer, where all code runs. Figure~\ref{fig:lab_setup} shows the lab experiment setup.

The simulations are run on a Dell XPS 15 laptop with a 12-core 2.60GHz Intel i7-10750H CPU.

\begin{figure}[!h]
    \centering
    \includegraphics[width=\columnwidth]{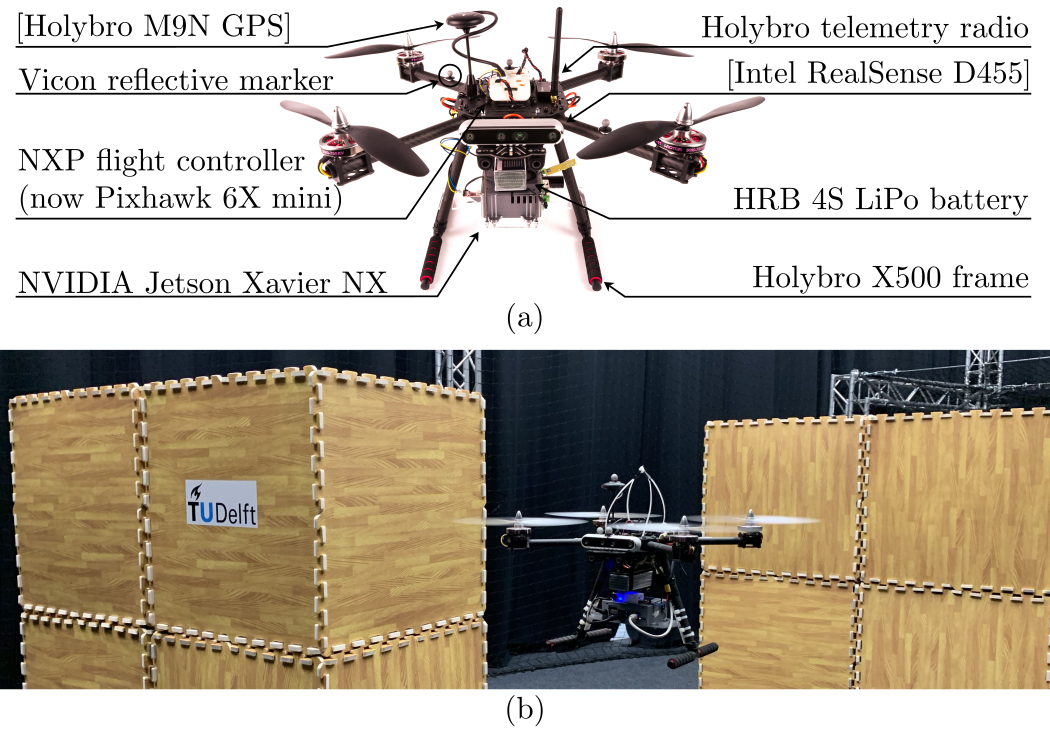}
    \caption{Lab experiment setup. (a) The quadrotor with important components is highlighted. Components between square brackets are not used in this work's experiments. (b) The quadrotor is flying near the obstacles in the lab.}
    \label{fig:lab_setup}
\end{figure}

\subsection{Results}\label{subsec:results_results}
This section presents simple simulation results focusing on the computational efficiency and size of the obstacle avoidance constraint sets in Section~\ref{subsubsec:results_results_idecomp} and the effect of the tuning parameter $\beta$ in Section~\ref{subsubsec:results_results_beta}. Thereafter, it describes the performance of the HMPC framework in lab experiments and simulations in Section~\ref{subsubsec:results_results_hmpc}.

A video of the lab experiments and simulations is available at https://youtu.be/0RnrKk6830I.

\subsubsection{\textit{I-DecompUtil}}\label{subsubsec:results_results_idecomp}
The purpose of this section is twofold. First, it shows how to make the runtime code for obstacle avoidance constraint generation efficient. Second, it demonstrates the impact of the bounding box width on the computation time and the time required to reach the goal.

It is important to note that \textit{DecompUtil} iterates through grid map $\mathcal{M}$ to find all occupied grid cells $\tilde{\mathcal{O}}$ that are contained in the bounding box $\mathcal{B}$ around the line segment. After finding the occupied grid cells, \textit{DecompUtil} iterates through them to compute the constraints for a specific line segment. \textit{I-DecompUtil} repeats this process for all line segments between the stages in the horizon.

The time to compute the obstacle avoidance constraints varies significantly depending on the implementation. Table~\ref{tab:con_time_results} shows the mean and standard deviation of the timing results for the different combinations of two options to reduce computation time: \textbf{C++ compiler optimization}, i.e., setting compiler flag \texttt{-O3} to allow for loop unrolling, and \textbf{map pre-processing}, i.e., selecting only the relevant part corresponding to the line segment to reduce the number of occupied grid cells to iterate through. The results are obtained by moving from start to goal using an MPC with a horizon of 100 stages and a sampling time of 50 ms. A long horizon is chosen so that the impact of other, high-priority processor tasks on the constraints computation time is averaged out.

\begin{table}[h!]
    \centering
    \caption{Mean (standard deviation) of constraint computation times (ms) over a single run from start to goal position.}
    \begin{tabular}{cc|c|c|}
        \cline{3-4}
        &&\multicolumn{2}{c|}{\textbf{Map pre-processing}}\\
        \cline{3-4}
        &&No&Yes\\
        \hline
        \multicolumn{1}{|c|}{\multirow{2}{*}{\textbf{Compiler optimization}}}&No&2312.6 (165.0)&147.7 (69.2)\\
        \cline{2-4}
        \multicolumn{1}{|c|}{}&Yes&6.1 (0.8)&5.5 (1.1)\\
        \hline
    \end{tabular}
    \label{tab:con_time_results}
\end{table}

The results clearly illustrate that compiler optimization is the main factor in reducing constraints computation time, allowing for real-time feasibility. The larger the map, the more significant the impact of map pre-processing becomes, up to the point that map pre-processing becomes necessary to ensure real-time feasibility.

Another factor to consider is the bounding box width tuning parameter. This parameter affects the size of the bounding box $\mathcal{B}$ and, correspondingly, the size of the pre-processed map. The primary purpose of $\mathcal{B}$ is to limit the number of grid cells to consider for constraints computation. As a result, this parameter trades off the speed at which the MPC scheme converges to the goal, or can find a way to the goal at all, and the constraints computation time.

Figure~\ref{fig:sim_con_bbox_timing} shows the computation times for four different bounding box widths, given the scenario with two obstacles described in Section~\ref{subsubsec:results_results_hmpc}. The figure illustrates that the computation time grows with the bounding box width but saturates after a certain width, depending on the obstacle shapes and density in the environment. The variation in timing depends on the environment as well.

\begin{figure}[t]
    \centering
    \includegraphics{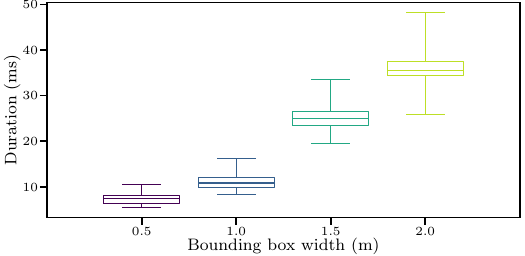}
    \caption{Obstacle avoidance constraints computation times for different bounding box widths when moving from start to goal with an MPC with 100 stages and 0.05 s sampling time.}
    \label{fig:sim_con_bbox_timing}
\end{figure}

Table~\ref{tab:sim_con_bbox_runs} shows the time required to move from start to goal for SMPC and HMPC using different bounding box widths. The SMPC has such a short horizon that the size of the obstacle avoidance constraints does not influence the goal-reaching time. On the other hand, for HMPC, a short bounding box width of 0.5 m limits the goal-reaching time. This effect vanishes in our environment setup for bounding box widths $\{1,1.5,2\}$ m. Instead, the specific shape of the constraint regions and the ability of the PMPC to plan more aggressive maneuvers that cause a slight goal overshooting play a more significant role in the goal-reaching times for these bounding box widths.

\begin{table}[h!]
    \centering
    \caption{Time to reach goal with SMPC and HMPC for different bounding box widths.}
    \begin{tabular}{cc|c|c|c|c|}
        \cline{3-6}
        &&\multicolumn{4}{c|}{$\bm{\mathcal{B}}$ \textbf{width (m)}}\\
        \cline{3-6}
        &&0.5&1&1.5&2\\
        \hline
        \multicolumn{1}{|c|}{\multirow{2}{*}{\textbf{Time to reach goal (s)}}}&SMPC&52.3&52.3&52.3&52.3\\
        \cline{2-6}
        \multicolumn{1}{|c|}{}&HMPC&14.8&6.4&7.3&6.9\\
        \hline
    \end{tabular}
    \label{tab:sim_con_bbox_runs}
\end{table}

The combination of the results in Figure~\ref{fig:sim_con_bbox_timing} and Table~\ref{tab:sim_con_bbox_runs} demonstrates the trade-off above up to 1 m: the smaller the bounding box, the less computation time but the more limited the plan is, and vice versa. Since a bounding box width of 1 m gives a similar goal reaching time to the larger bounding box widths but with lower computation times, this value is used to generate the rest of the results in this section.

It is important to note that the above results are generated using the grid map as presented in Section~\ref{subsubsec:results_results_hmpc}. This grid map is pre-generated with thin obstacle edge lines to make SMPC feasible in real time. However, in experiments with onboard perception, the grid map will contain thicker obstacle edges due to noise in sensor measurements (e.g., from a depth camera). Furthermore, these grid maps are usually constructed in 3D instead of 2D, which will significantly increase the number of occupied grid cells and enlarge the impact of map-preprocessing.

\subsubsection{Effect of $\beta$}\label{subsubsec:results_results_beta}
Next, we demonstrate the impact of $\beta$ on the PMPC. Recall from Remark~\ref{rem:beta} that $\beta$ trades off PMPC execution time and performance. This result is visible in Figure~\ref{fig:sim_beta}: the larger $\beta$, the shorter the PMPC execution time and the longer the goal-reaching time. The longer goal-reaching time is caused by the coarser re-planning time and the fact that the obstacle avoidance constraints are more conservative for a coarser plan, meaning that the PMPC can plan less far ahead in each optimization step.

\begin{figure}[t]
    \centering
    \includegraphics{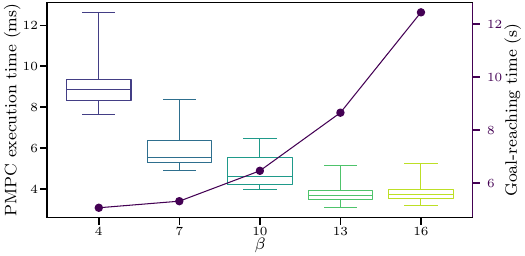}
    \caption{PMPC execution times (boxplots) and goal-reaching times (line) for different values of $\beta$ when moving from start to goal with HMPC. The TMPC is implemented as given in Table~\ref{tab:settings}. Given constant $T^\mathrm{s,t}$, the PMPC is implemented using $T^\mathrm{s,p}=\beta T^\mathrm{s,t}$ with $\beta \in \{4,7,10,13,16\}$. Except for $\beta=16$, $T^\mathrm{p}$ was chosen to be the closest to the value in Table~\ref{tab:settings}. For $\beta=16$, the horizon has to be $T^\mathrm{s,p}$ longer according to Remark~\ref{rem:beta}, which causes a slight increase in computation time compared to $\beta=13$.}
    \label{fig:sim_beta}
\end{figure}

Based on Figure~\ref{fig:sim_beta}, we have set $\beta=10$ in the next section. Note that $\beta=10$ allows for convenient results interpretation since $T^\mathrm{s,p}=T^\mathrm{t}$ given the numbers in Table~\ref{tab:settings}.

\subsubsection{HMPC}\label{subsubsec:results_results_hmpc}
This section describes the performance of HMPC compared to SMPC in lab experiments and highlights differences with simulation results when relevant. Two types of simulations are run: a simple simulation with perfect system knowledge, i.e., the simulated system is also \eqref{eq:f} integrated using RK4 with sampling time $T^{\mathrm{s},\mathrm{t}}$, and one in Gazebo, i.e., including model mismatch.

First, we need to know whether all MPC schemes are feasible on the embedded computer in real-time. Figure~\ref{fig:lab_smpc_hmpc_box} shows the execution times of the control loop and its main components for SMPC, PMPC, and TMPC. PMPC and TMPC comply with execution time limits, whereas SMPC is around its maximum of 50 ms and sometimes exceeds the maximum. This problem cannot easily be avoided on the embedded computer since it has a limited number of threads. The lower the number of threads, the greater the risk of the control loop being interrupted by other operating system tasks. Since most of the SMPC control loops finish in time, the frequency at which control loops are scheduled catches up with the desired frequency after exceeding the limit. Consequently, the results still provide valuable insights.

\begin{figure}[t]
    \centering
    \includegraphics{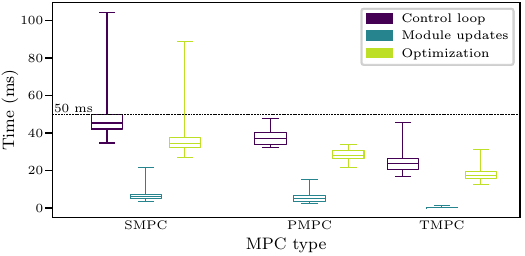}
    \caption{Boxplot with overall execution times for SMPC and HMPC (\textit{Control loop}) and its most important parts: \textit{Module updates} and \textit{Optimization}. The maximum control loop execution time is 50 ms for real-time feasibility of SMPC and TMPC and 500 ms for PMPC. Note that \textit{Module updates} includes constraints generation and loading new cost function terms, and is low compared to the optimization time. Since the TMPC receives the constraints from the PMPC, its \textit{Module updates} computation time is even lower. \textit{Optimization} is the time between calling the solver and obtaining the solution. SMPC contains several control loop cycles exceeding its real-time limit, whereas PMPC and TMPC always finish on time.}
    \label{fig:lab_smpc_hmpc_box}
\end{figure}

To show the impact of the embedded hardware on the execution times, Table~\ref{tab:sim_lab_timing} shows the mean and standard deviation of the control loop execution times for SMPC, PMPC, and TMPC in both simulations and lab experiments. The numbers clearly illustrate a drastic increase in mean and standard deviation of the control loop execution time when moving from simulations to lab experiments with embedded hardware.

\begin{table}[h!]
    \centering
    \caption{Mean (standard deviation) of control loop execution times (ms) of SMPC, PMPC, and TMPC in simple simulations (sim), Gazebo simulations (gaz), and lab experiments using embedded hardware (lab).}
    \begin{tabular}{cc|c|c|c|c|}
        \cline{3-5}
        &&\multicolumn{3}{c|}{\textbf{System}}\\
        \cline{3-5}
        &&sim&gaz&lab\\
        \hline
        \multicolumn{1}{|c|}{\multirow{3}{*}{\textbf{Method}}}&SMPC&8.0 (1.2)&8.4 (1.3)&\textbf{46.8} (7.3)\\
        \cline{2-5}
        \multicolumn{1}{|c|}{}&PMPC&4.9 (0.9)&5.2 (0.9)&\textbf{37.7} (4.2)\\
        \cline{2-5}
        \multicolumn{1}{|c|}{}&TMPC&3.2 (0.6)&3.4 (0.7)&\textbf{24.1} (4.9)\\
        \hline
    \end{tabular}
    \label{tab:sim_lab_timing}
\end{table}

The main limitation of SMPC is the computation time. Therefore, an SMPC with increasing discretization times was also tested in an attempt to reduce the number of stages while keeping the same planning horizon length. However, this scheme is not able to plan a trajectory around an obstacle as the line segments of two stages separated by a relatively large discretization time intersect with the obstacles when moving around them, i.e. Assumption~\ref{ass:map} does not hold.

Note that since the PMPC has a sampling time of 500 ms and only needs 50 ms to solve the considered scenario, it has time left to account for processing a more realistic grid map and optimizing a plan in a more challenging environment.

Given the real-time feasibility of SMPC and HMPC, we can now analyze the schemes' closed-loop properties. Table~\ref{tab:sim_lab_goal_times} summarizes the goal-reaching times of SMPC and HMPC in all simulations and lab experiments. In this case, the goal is reached if the system enters the circle with a radius of 5 cm around the goal in $(p^x,p^y)$-plane. We can draw three conclusions based on the results. First, HMPC can move more efficiently from start to goal. Second, the presence of model mismatch significantly affects the goal-reaching times of SMPC, while HMPC is less sensitive to the presence of model mismatch. Third, the goal-reaching time of SMPC in lab experiments is significantly shorter than in simulations. Compared to simple simulations, the quadrotor moves quicker around the obstacle edges since the combination of model mismatch and slack results in higher accelerations, also in the direction of moving alongside the obstacle. Compared to the Gazebo simulations, the real quadrotor responds more quickly to control commands. This effect becomes more significant when the cost function gives less incentive to increase attitude, and consequently, acceleration. Therefore, the real quadrotor reaches the circle with a radius of 5 cm around the goal faster.

\begin{table}[h!]
    \centering
    \caption{Time (s) to move from start to goal for SMPC and HMPC in simple simulations (sim), Gazebo simulations (gaz), and lab experiments (lab).}
    \begin{tabular}{cc|c|c|c|c|}
        \cline{3-5}
        &&\multicolumn{3}{c|}{\textbf{System}}\\
        \cline{3-5}
        &&sim&gaz&lab\\
        \hline
        \multicolumn{1}{|c|}{\multirow{2}{*}{\textbf{Method}}}&SMPC&52.3&107.5&35.1\\
        \cline{2-5}
        \multicolumn{1}{|c|}{}&HMPC&6.4&8.4&8.1\\
        \hline
    \end{tabular}
    \label{tab:sim_lab_goal_times}
\end{table}

Figure~\ref{fig:lab_smpc_hmpc_traj} shows the 2D closed-loop trajectories of both SMPC and HMPC in the lab experiments. Furthermore, Figure~\ref{fig:lab_pmpc_smpc_costs} displays the corresponding planning costs of SMPC and PMPC.

\begin{figure}[t]
    \centering
    \includegraphics{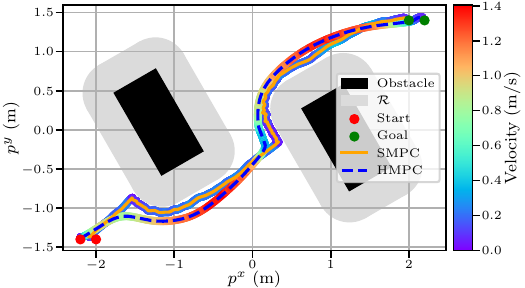}
    \caption{2D closed-loop SMPC and HMPC trajectories with associated velocities in lab experiments. HMPC has a longer planning horizon and results in a smoother trajectory reaching the goal in 8.1 s versus 35.1 s for SMPC. Note that the start and goal positions for SMPC are closer. Otherwise, SMPC cannot find a way around the first encountered obstacle.}
    \label{fig:lab_smpc_hmpc_traj}
\end{figure}

\begin{figure}[t]
    \centering
    \includegraphics{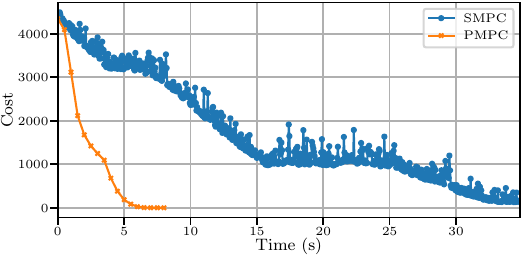}
    \caption{Optimal costs for SMPC and PMPC runs. Due to its longer horizon, PMPC reduces its cost quickly, whereas SMPC takes longer to move around obstacles and is sensitive to model mismatch.}
    \label{fig:lab_pmpc_smpc_costs}
\end{figure}

Related to the SMPC performance, we would like to highlight three aspects.

First of all, the \textbf{SMPC} \textbf{start and goal positions are closer} to each other compared to HMPC. With the original start and goal positions, SMPC gets stuck behind the first encountered obstacle. This is due to the short horizon that ends in steady state and the fact that SMPC is a local method. After moving the start and goal positions, SMPC still struggles to find a way around the obstacles because of its short horizon. As a result, the trajectory slowly rotates around the obstacle corners, and it takes a long time to reach the goal. This is reflected in Figure~\ref{fig:lab_pmpc_smpc_costs} by the `plateaus' in the cost function.

Secondly, SMPC is \textbf{sensitive to model mismatch}. Without constraints softening, a slight disturbance effect causes the system to exceed the constraints, making the optimization problem infeasible. With constraints softening, the system gets aggressively pushed back due to the high slack penalization term in the cost function. However, this increases the distance to the goal, causing the robot to move back towards the constraints, repeating the same process. This is visible in Figure~\ref{fig:lab_smpc_hmpc_traj} by the oscillating trajectory near the obstacles and in Figure~\ref{fig:lab_pmpc_smpc_costs} by the noisy cost values.

To prevent collisions despite this oscillating behavior, the safety margin, as described in Section~\ref{subsubsec:results_impl_slack}, is tuned to 0.1 m. Figure~\ref{fig:lab_smpc_slack} shows the maximum slack value that occurs in the prediction horizon of the different SMPC runs over time. The maximum value is 0.15 m, meaning that the predicted position in at least one of the stages exceeds the safety margin of 0.1 m. This happens in the last stage, not in the first stage, meaning the robot does not collide with the obstacles. However, the SMPC scheme cannot provide strict safety guarantees in general.

\begin{figure}[t]
    \centering
    \includegraphics{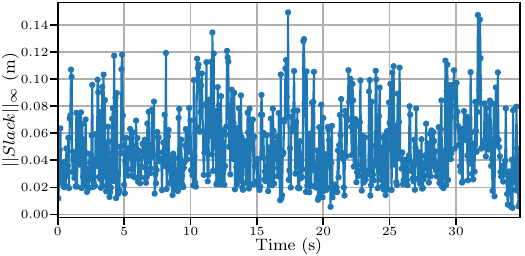}
    \caption{Maximum SMPC slack in the horizon over time. The maximum slack value is around 0.15 m.}
    \label{fig:lab_smpc_slack}
\end{figure}

Finally, SMPC is \textbf{not able to maintain altitude} during flight. This is caused by the fact that the SMPC needs to balance both the planning task, reaching the goal without collisions, and the tracking task, sending control commands to achieve stable flight, in a single formulation with a short horizon and in the presence of model mismatch, especially in thrust dynamics. In simple simulation, the $z$ position varies between 1.39 and 1.41 m, in Gazebo between 1.37 and 1.67 m, and in lab experiments between 0.20 and 1.56 m.

Next, we would like to highlight six aspects related to the performance of \textbf{HMPC}.

First of all, to verify the implementation of the HMPC scheme, we checked the simple simulation results and concluded that the TMPC tracking error was zero at all times. This proves the \textbf{validity of Theorems~\ref{thm:tmpc} and~\ref{thm:pmpc}}.

Secondly, in contrast to SMPC, HMPC gives a similar closed-loop trajectory and similar goal-reaching times in simple simulations, Gazebo simulations, and lab experiments. Therefore, it is \textbf{less sensitive to model mismatch compared to SMPC}.

Thirdly, the PMPC horizon is long enough to \textbf{plan a reasonable trajectory around the obstacles}. This is clearly shown by the quickly decreasing PMPC cost compared to the SMPC cost in Figure~\ref{fig:lab_pmpc_smpc_costs}. Due to its longer horizon, PMPC quickly finds a way towards the goal. On the other hand, SMPC needs time to move around the obstacles. The horizontal parts of the SMPC cost illustrate this effect. Moreover, the SMPC cost is noisier, which is caused by the model mismatch.\looseness=-1

Fourthly, the TMPC has a sufficiently long horizon to \textbf{satisfy the terminal set constraint in every run}, as is clearly illustrated for a particular run between the obstacles in Figure~\ref{fig:lab_tmpc_prediction}.

\begin{figure}[t]
    \centering
    \includegraphics{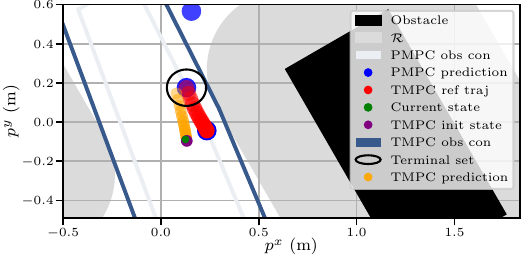}
    \caption{A specific snapshot of the \textit{TMPC prediction} starting from a forward-simulated state (\textit{TMPC init state}) and ending in terminal set \eqref{eq:term_set}, given the obstacles, their collision region indicated by the inflation with the radius of the robot region $\mathcal{R}$, \textit{PMPC prediction} as reference plan and corresponding reference trajectory (\textit{TMPC ref traj}), and obstacle avoidance constraints (\textit{PMPC obs con}). Note that the PMPC obstacle avoidance constraints are tightened with respect to the TMPC obstacle avoidance constraints (\textit{TMPC obs con}). Furthermore, the last measured state (\textit{Current state}) is plotted. The last measured state and forward-simulated state do not overlap, showing the presence of model mismatch.}
    \label{fig:lab_tmpc_prediction}
\end{figure}

Fifthly, HMPC can \textbf{maintain altitude} during flight: in simple simulations, the $z$ position varies between 1.37 and 1.40 m, in Gazebo between 1.39 and 1.42 m, and in lab experiments 1.32 and 1.43 m.

Finally, to show that HMPC \textbf{scales to more complex environments}, Figure~\ref{fig:gaz_hmpc_traj} illustrates the closed-loop trajectory of HMPC in a corridor-like Gazebo environment. Even though HMPC is a local method, it finds its way from start to goal in 8.4 s without colliding with the obstacles. The maximum PMPC and TMPC control loop execution times are 7.0 ms and 5.8 ms, respectively, meaning that HMPC is \textbf{runtime feasible}.

\begin{figure}[t]
    \centering
    \includegraphics{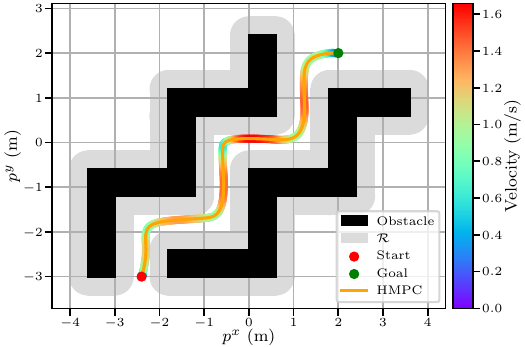}
    \caption{2D closed-loop HMPC trajectory with associated velocity in a more complex Gazebo environment. HMPC can deal with more complex environments and finds its way to the goal in 8.4 s without collisions.}
    \label{fig:gaz_hmpc_traj}
\end{figure}

\section{Conclusions and future work}\label{sec:conclusion}
Optimizing a trajectory complying with nonlinear system dynamics and avoiding collisions is a computationally expensive task. Consequently, it is not feasible to run a single-layer MPC (SMPC) scheme in real time that generates obstacle avoidance constraints and optimizes the corresponding trajectory with a reasonable horizon length on an embedded computer.

To address this problem, we propose a novel hierarchical MPC (HMPC) scheme, which includes the co-design of a planning MPC (PMPC) and tracking MPC (TMPC). In the offline phase, we compute the terminal ingredients of the TMPC. These are also utilized to adjust the constraints in the PMPC during runtime. By ensuring the continuity of the reference plan and consistency in updating the obstacle avoidance constraints, the PMPC can run independently from the current system state and construct a plan with a long time horizon. This saves constraints generation and optimization time on the side of the TMPC. Consequently, the TMPC runs at a higher frequency, resulting in accurate tracking of the reference plan.

We compared HMPC to SMPC in simulations without model mismatch, Gazebo simulations, and lab experiments. In general, the prediction horizon of SMPC needs to be reduced significantly to achieve real-time feasibility in the lab experiments. Therefore, SMPC has difficulty finding a way to the goal, needs constraints softening to remain feasible, and deviates significantly from the desired altitude. On the other hand, HMPC optimizes a smoother plan, does not require constraints softening, and maintains altitude more accurately.

To summarize, HMPC is more computationally efficient and less sensitive to model mismatch than SMPC. HMPC also provides tracking and recursive feasibility guarantees. It simplifies the task for robotic practitioners so they can leverage off-the-shelf nonlinear solvers for motion planning and tracking without having to optimize code, linearize models, or take care of safety themselves. Furthermore, due to its computational efficiency, HMPC can easily be deployed on different robotic platforms.

While the experiments and simulations demonstrate collision avoidance and recursive feasibility, the derived theoretical guarantees are only valid in the absence of model mismatch. Our next step is to extend the TMPC to a \textit{robust} formulation that can explicitly leverage pre-determined bounds on the model mismatch, similar to the ideas in \cite{kohler2020computationally}. Lastly, we would like to move away from motion capture cameras to onboard sensors, such as LiDAR or depth cameras, for obstacle detection and 3D constraint generation, e.g., in outdoor scenarios.

{\appendices

\section{Proof Theorem~\ref{thm:tmpc}}\label{app:theorem_tmpc}
\begin{proof}
    The proof consists of three parts. \textbf{Part I} proposes the previously optimal solution shifted by $T^\mathrm{s}=T^{\mathrm{s},\mathrm{t}}$, appended with the state after applying $\kappa^\mathrm{f}$, as candidate solution with horizon $T=T^\mathrm{t}$ for the next TMPC run. \textbf{Part II} shows that all system and obstacle constraints are satisfied for the candidate solution. Finally, \textbf{Part III} shows asymptotic convergence of the state to the reference trajectory.

    \textbf{Part I. Candidate solution}\\
    For convenience, define for $\tau \in [T,T+T^\mathrm{s}]$:
    \begin{equation}\label{eq:tmpc_candidate_term_u}
        \bm{u}_{\tau|t}^* \coloneqq \kappa^\mathrm{f}(\bm{x}_{\tau|t}^*,\bm{r}_{\tau|t}^\mathrm{r}),\\
    \end{equation}
    and $\bm{x}_{\tau|t}^*$ according to \eqref{eq:tmpc_state_update} by applying \eqref{eq:tmpc_candidate_term_u}.

    Consider the following candidate solution:
    \begin{equation}\label{eq:tmpc_candidate}
        \tilde{\bm{u}}_{\tau|t+T^\mathrm{s}} = \bm{u}_{\tau+T^\mathrm{s}|t}^*, \ \tilde{\bm{x}}_{\tau|t+T^\mathrm{s}} = \bm{x}_{\tau+T^\mathrm{s}|t}^*, \ \tau \in [0,T],
    \end{equation}
    with the reference trajectory defined by \eqref{eq:property_ref_traj_update} in Property~\ref{property:ref_traj}.

    \textbf{Part II. Recursive feasibility}\\
    \textit{Part II-I. System constraints satisfaction} For $\tau \in [0,T-T^\mathrm{s}), j \in \mathbb{N}_{[1,n^\mathrm{s}]}$, the system constraints given the candidate solution satisfy:
    \begin{equation}
        g_{j}^\mathrm{s}(\tilde{\bm{x}}_{\tau|t+T^\mathrm{s}},\tilde{\bm{u}}_{\tau|t+T^\mathrm{s}}) \overset{\eqref{eq:tmpc_candidate}}{=} g_{j}^\mathrm{s}(\bm{x}_{\tau+T^\mathrm{s}|t}^*,\bm{u}_{\tau+T^\mathrm{s}|t}^*) \overset{\eqref{eq:tmpc_sys_constraints}}{\leq} 0.
    \end{equation}

    \textit{Part II-II. Obstacle avoidance} Similarly, by Property~\ref{property:obst_constraints}~\ref{property:obst_constraints_update}, for $\tau \in (0,T-T^\mathrm{s}], j \in \mathbb{N}_{[1,n^\mathrm{o}]}$, the obstacle avoidance constraints given the candidate satisfy:
    \begin{equation}\label{eq:tmpc_proof_obst_constraints}
        \begin{aligned}
            &g_{j, \tau+T^\mathrm{s}|t}^\mathrm{o}(\bm{p}_{\tau+T^\mathrm{s}|t}^*) \overset{\eqref{eq:tmpc_obst_constraints}\eqref{eq:obst_constraints}}{\leq} 0\\
            \overset{\eqref{eq:tmpc_candidate}}{\implies} &g_{j, \tau|t+T^\mathrm{s}}^\mathrm{o}(\tilde{\bm{p}}_{\tau|t+T^\mathrm{s}}) \leq 0,
        \end{aligned}
    \end{equation}
    thus ensuring obstacle avoidance by Property~\ref{property:obst_constraints}~\ref{property:obst_constraints_construction}.

    \textit{Part II-III. Terminal constraints satisfaction} Given candidate \eqref{eq:tmpc_candidate} and reference \eqref{eq:property_ref_traj_update}, the following holds for $\tau \in [T-T^\mathrm{s},T]$:
    \begin{align}
        \mathcal{J}^\mathrm{f,t}(\tilde{\bm{x}}_{\tau|t+T^\mathrm{s}},\bm{x}_{\tau|t+T^\mathrm{s}}^\mathrm{r}) \overset{\eqref{eq:tmpc_candidate}}{=} &\mathcal{J}^\mathrm{f,t}(\bm{x}_{\tau+T^\mathrm{s}|t}^*,\bm{x}_{\tau|t+T^\mathrm{s}}^\mathrm{r}) \notag\\
        \overset{\eqref{eq:prop_term_ing_jf_decrease}}{\leq} &\mathcal{J}^\mathrm{f,t}(\bm{x}_{T|t}^*,\bm{x}_{T|t}^\mathrm{r}) \overset{\eqref{eq:tmpc_term_set}}{\leq} \alpha^2,
    \end{align}
    i.e. the terminal set $(\tilde{\bm{x}}_{\tau|t+T^\mathrm{s}}{-}\bm{x}_{\tau|t+T^\mathrm{s}}^\mathrm{r}) \in \mathcal{X}^{\mathrm{t},\mathrm{f}}$ is positive invariant. Therefore, by Proposition~\ref{prop:term_ing}, the following holds for $\tau \in [T-T^\mathrm{s},T]$:
    \begin{align}
        g_{j}^\mathrm{s}(\tilde{\bm{x}}_{\tau|t+T^\mathrm{s}},\tilde{\bm{u}}_{\tau|t+T^\mathrm{s}}) \overset{\eqref{eq:prop_term_ing_sys_constraints}}{\leq} &0,\ &&j \in \mathbb{N}_{[1,n^\mathrm{s}]},\\
        g_{j, \tau|t+T^\mathrm{s}}^\mathrm{o}(\tilde{\bm{p}}_{\tau|t+T^\mathrm{s}}) \overset{\eqref{eq:prop_term_ing_obst_constraints}}{\leq} &0,\ &&j \in \mathbb{N}_{[1,n^\mathrm{o}]}.
    \end{align}

    \textbf{Part III. Asymptotic convergence}\\
    Asymptotic convergence is shown using Barbalat's Lemma. By \eqref{eq:prop_term_ing_jf_decrease}, the following holds:
    \begin{align}\label{eq:proof_jf_decrease}
        &\mathcal{J}^\mathrm{f,t}(\tilde{\bm{x}}_{T|t+T^\mathrm{s}}{-}\bm{x}_{T|t+T^\mathrm{s}}^\mathrm{r}) - \mathcal{J}^\mathrm{f,t}(\bm{x}_{T|t}^*{-}\bm{x}_{T|t}^\mathrm{r}) \notag\\&+ \int_{\tau=T-T^\mathrm{s}}^T \mathcal{J}^\mathrm{s,t}(\bm{x}_{\tau+T^\mathrm{s}|t}^*,\bm{x}_{\tau+T^\mathrm{s}|t}^*,\bm{r}_{\tau+T^\mathrm{s}|t}) d\tau \leq 0.
    \end{align}
    Therefore, the following inequality is established:
    \begin{align}
        &\mathrlap{\mathcal{J}^{*,\mathrm{t}}(\bm{x}_{t+T^\mathrm{s}},\bm{x}_{\cdot|t+T^\mathrm{s}}^\mathrm{r})}&& \notag\\
        &\leq &&\int_{\tau=0}^T \mathcal{J}^\mathrm{s,t}(\tilde{\bm{x}}_{\tau|t+T^\mathrm{s}},\tilde{\bm{u}}_{\tau|t+T^\mathrm{s}},\bm{r}_{\tau|t+T^\mathrm{s}}) d\tau \notag\\&&&+\mathcal{J}^\mathrm{f,t}(\tilde{\bm{x}}_{T|t+T^\mathrm{s}}{-}\bm{x}_{T|t+T^\mathrm{s}}^\mathrm{r}) \notag\\
        &\overset{\eqref{eq:proof_jf_decrease}}{\leq} &&\mathcal{J}^{*,\mathrm{t}}(\bm{x}_t,\bm{x}_{\cdot|t}^\mathrm{r}) - \int_{\tau=0}^{T^\mathrm{s}} \mathcal{J}^\mathrm{s,t}(\bm{x}_{\tau|t}^*,\bm{x}_{\tau|t}^*,\bm{r}_{\tau|t}) d\tau,
    \end{align}
    which proves:
    \begin{multline}\label{eq:proof_convergence}
        \mathcal{J}^{*,\mathrm{t}}(\bm{x}_{t+T^\mathrm{s}},\bm{x}_{\cdot|t+T^\mathrm{s}}^\mathrm{r}) - \mathcal{J}^{*,\mathrm{t}}(\bm{x}_t,\bm{x}_{\cdot|t}^\mathrm{r})\\ \leq -c^{\mathcal{J},\mathrm{d}} \int_{\tau=t}^{t+T^\mathrm{s}} \norm{\bm{x}_t - \bm{x}_t^\mathrm{r}}_Q^2 d\tau,
    \end{multline}
    using $Q,R \succ 0$ with $c^{\mathcal{J},\mathrm{d}} > 0$.
    Iterating this inequality and using the fact that $\mathcal{J}^{*,\mathrm{t}}(\bm{x}_t,\bm{x}_{\cdot|t}^\mathrm{r})$ is uniformly bounded yields:
    \begin{align}
        \lim_{t\rightarrow\infty}\int_{0}^t \|\bm{x}_\tau-\bm{x}_\tau^{\mathrm{r}}\|_Q^2d\tau<\infty.
    \end{align}
    Asymptotic convergence follows by invoking Barbalat's Lemma.
\end{proof}

\section{Proof Theorem~\ref{thm:pmpc}}\label{app:theorem_pmpc}
\begin{proof}
    The proof consists of two parts. \textbf{Part I} proposes the previously optimal solution shifted by $T^\mathrm{s}=T^{\mathrm{s},\mathrm{p}}$, appended with an input ensuring steady state, as candidate solution with horizon $T=T^\mathrm{p}$ for the next PMPC run. \textbf{Part II} shows feasibility of the candidate solution given the previously optimal solution.

    \textbf{Part I. Candidate solution}\\
    Consider the following candidate solution, valid by \eqref{eq:pmpc_term_set}:
    \begin{subequations}\label{eq:pmpc_candidate}
        \begin{align}
            \tilde{\bm{u}}_{\tau|t+T^\mathrm{s}} = 
            \begin{cases}
                \bm{u}_{\tau+T^\mathrm{s}|t}^*, &\tau \in [0,T-T^\mathrm{s}),\\
                \bm{u}_{T|t}^*, &\tau \in [T-T^\mathrm{s},T],
            \end{cases} \label{eq:pmpc_candidate_u}\\
            \tilde{\bm{x}}_{\tau|t+T^\mathrm{s}} = 
            \begin{cases}
                \bm{x}_{\tau+T^\mathrm{s}|t}^*, &\tau \in [0,T-T^\mathrm{s}),\\
                \bm{x}_{T|t}^*, &\tau \in [T-T^\mathrm{s},T].
            \end{cases} \label{eq:pmpc_candidate_x}
        \end{align}
    \end{subequations}

    \textbf{Part II. Recursive feasibility}\\
    \textit{Part II-I. System constraints satisfaction} For $\tau \in [0,T-T^\mathrm{s}), j \in \mathbb{N}_{[1,n^\mathrm{s}]}$, the system constraints given the candidate solution are:
    \begin{equation}\label{eq:pmpc_candidate_sys_constraints}
        g_{j}^\mathrm{s}(\tilde{\bm{x}}_{\tau|t+T^\mathrm{s}},\tilde{\bm{u}}_{\tau|t+T^\mathrm{s}}) \overset{\eqref{eq:pmpc_candidate}}{=} g_{j}^\mathrm{s}(\bm{x}_{\tau+T^\mathrm{s}|t}^*,\bm{u}_{\tau+T^\mathrm{s}|t}^*) \overset{\eqref{eq:pmpc_sys_constraints}}{\leq} 0.
    \end{equation}

    \textit{Part II-II. Obstacle avoidance constraints satisfaction} Similarly, by Property~\ref{property:obst_constraints}~\ref{property:obst_constraints_update} for $\tau \in (0,T-T^\mathrm{s}], j \in \mathbb{N}_{[1,n^\mathrm{o}]}$, the obstacle avoidance constraints given the candidate satisfy:
    \begin{equation}\label{eq:pmpc_candidate_obst_constraints}
        \begin{aligned}
            &g_{j, \tau+T^\mathrm{s}|t}^\mathrm{o}(\bm{p}_{\tau+T^\mathrm{s}|t}^*) \overset{\eqref{eq:pmpc_obst_constraints}\eqref{eq:obst_constraints}\eqref{eq:property_ref_traj_obst_constraints_tightened}}{\leq} 0\\
            \overset{\eqref{eq:pmpc_candidate_x}}{\implies} &g_{j, \tau|t+T^\mathrm{s}}^\mathrm{o}(\tilde{\bm{p}}_{\tau|t+T^\mathrm{s}}) \leq 0.
        \end{aligned}
    \end{equation}

    \textit{Part II-III. Terminal constraints satisfaction} Given the steady-state condition in candidate \eqref{eq:pmpc_candidate}, \eqref{eq:pmpc_candidate_sys_constraints} and \eqref{eq:pmpc_candidate_obst_constraints} hold for $\tau \in [T-T^\mathrm{s},T]$ as well.
\end{proof}

}

\section*{Acknowledgments}
The authors would like to thank Chris Pek for helping to create Figure~\ref{fig:smpc_hmpc_front}, and Clarence Chen, Johanna Probst, and Simone Carena for helping to take the pictures in Figure~\ref{fig:lab_setup}.

\bibliographystyle{IEEEtran}
\bibliography{IEEEabrv,mybib.bib}

\begin{IEEEbiography}[{\includegraphics[width=1in,height=1.25in,clip,keepaspectratio]{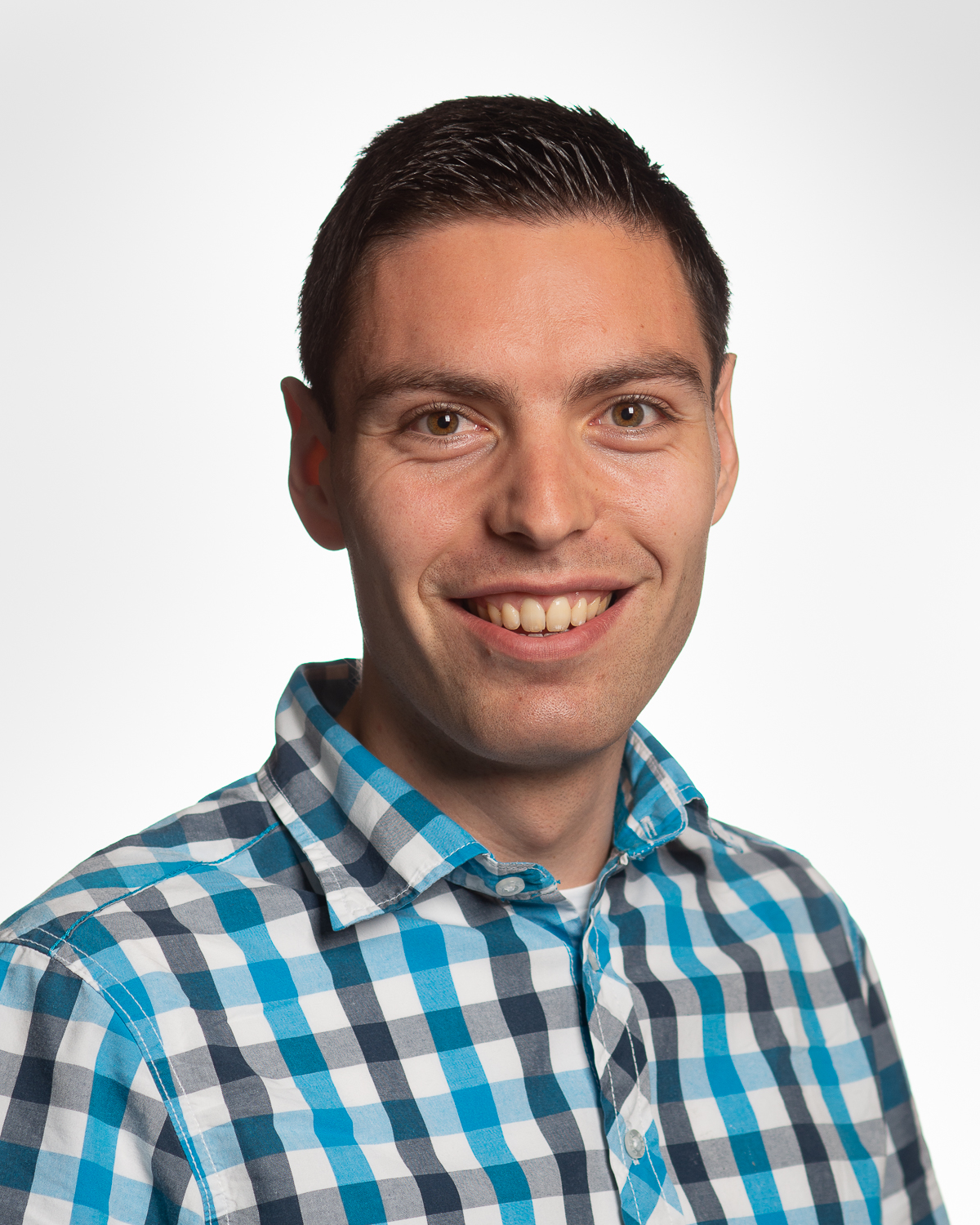}}]{Dennis Benders} received the B.Sc. degree in Electrical Engineering and M.Sc. degree in Embedded Systems from Delft University of Technology, Delft, The Netherlands, in 2018 and 2020, respectively. He is currently working towards a Ph.D. degree in robotics. He is part of the Reliable Robot Control Lab and Autonomous Multi-Robots Lab at Delft University of Technology. His research interests include motion planning and control for autonomous mobile robots, with special emphasis on real-time and open-source code.
\end{IEEEbiography}

\begin{IEEEbiography}[{\includegraphics[width=1in,height=1.25in,clip,keepaspectratio]{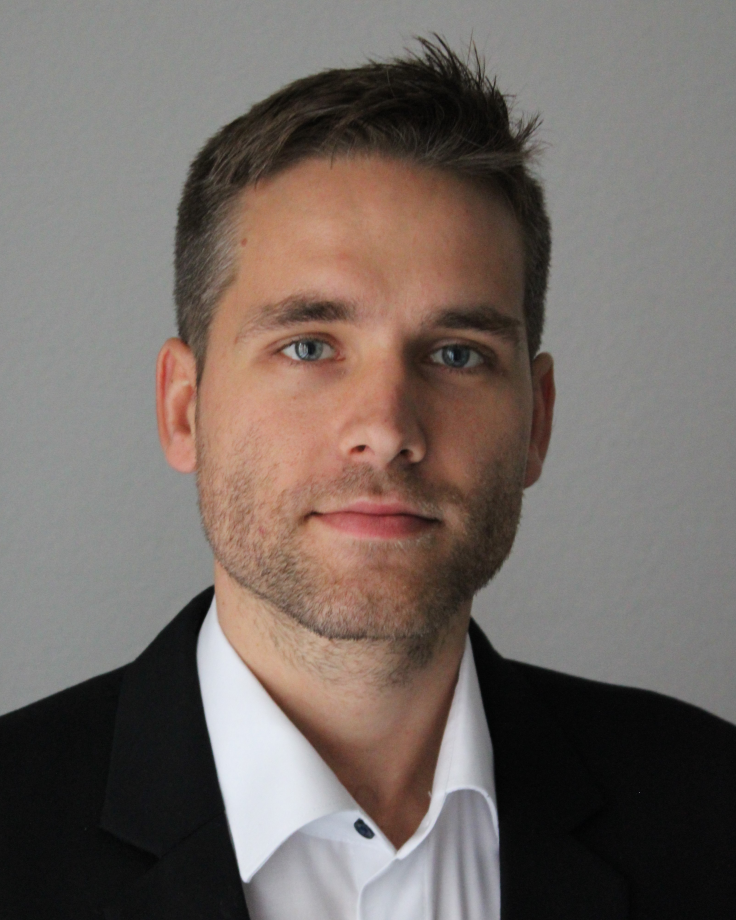}}]{Johannes K\"ohler} received the Ph.D. degree from the University of Stuttgart, Germany, in 2021. He is currently a postdoctoral researcher at ETH Z\"{u}rich, Switzerland. His research interests are in the area of model predictive control and control and estimation for nonlinear uncertain systems. Dr. K\"{o}hler is the recipient of the 2021 European Systems \& Control PhD Thesis Award, the IEEE CSS George S. Axelby Outstanding Paper Award 2022, and the Journal of Process Control Paper Award 2023.
\end{IEEEbiography}

\begin{IEEEbiography}[{\includegraphics[width=1in,height=1.25in,clip,keepaspectratio]{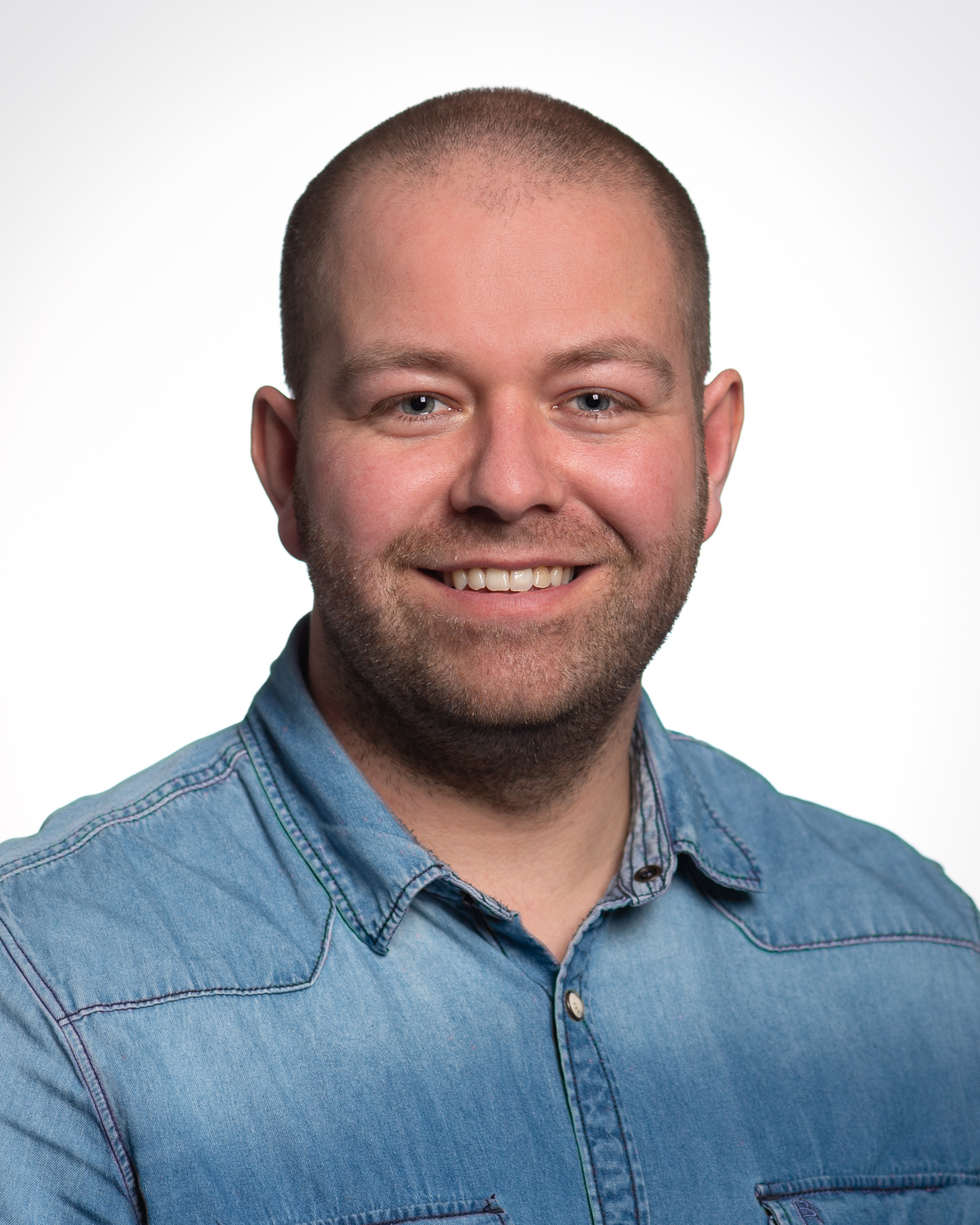}}]{Thijs Niesten} received the B.Sc. degree and M.Sc. degree in Mechanical Engineering from Delft University of Technology, Delft, The Netherlands, in 2018 and 2022, respectively. He is currently working as a research engineer in robotics. He is part of the Reliable Robot Control Lab and Autonomous Multi-Robots Lab at Delft University of Technology. His research interests include computer vision, motion planning, and control for autonomous mobile robots.
\end{IEEEbiography}

\begin{IEEEbiography}[{\includegraphics[width=1in,height=1.25in,clip,keepaspectratio]{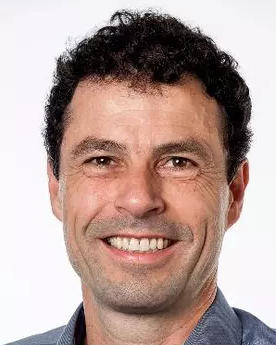}}]{Robert Babu{\v s}ka} received the M.Sc. (Hons.) degree from the Czech Technical University in Prague, in 1990, and the Ph.D. (cum laude) degree from Delft University of Technology, the Netherlands, in 1997. He is a full professor at TU Delft, Faculty of Mechanical Engineering, Department of Cognitive Robotics, and a distinguished researcher at Czech Technical University in Prague, CIIRC, Department of AI. His research interests include reinforcement learning, adaptive and learning robot control, and nonlinear system identification. He has been involved in the applications of these techniques in various fields, ranging from process control to robotics and aerospace.
\end{IEEEbiography}

\begin{IEEEbiography}[{\includegraphics[width=1in,height=1.25in,clip,keepaspectratio]{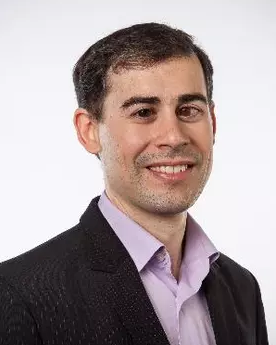}}]{Javier Alonso-Mora} is an associate professor at the Cognitive Robotics department of the Delft University of Technology, where he leads the Autonomous Multi-Robots Lab. Dr. Alonso-Mora was a postdoctoral associate at the Massachusetts Institute of Technology (MIT). He received his Ph.D. degree in robotics from ETH Z\"urich. His research focuses on navigation, motion planning, learning, and control of autonomous mobile robots and teams, interacting with humans in dynamic, uncertain environments. He is the recipient of multiple awards, including the IEEE T-ASE Best Paper Award (2024), an ERC Starting Grant (2021), and the ICRA Best Paper Award on Multi-Robot Systems (2019).
\end{IEEEbiography}

\begin{IEEEbiography}[{\includegraphics[width=1in,height=1.25in,clip,keepaspectratio]{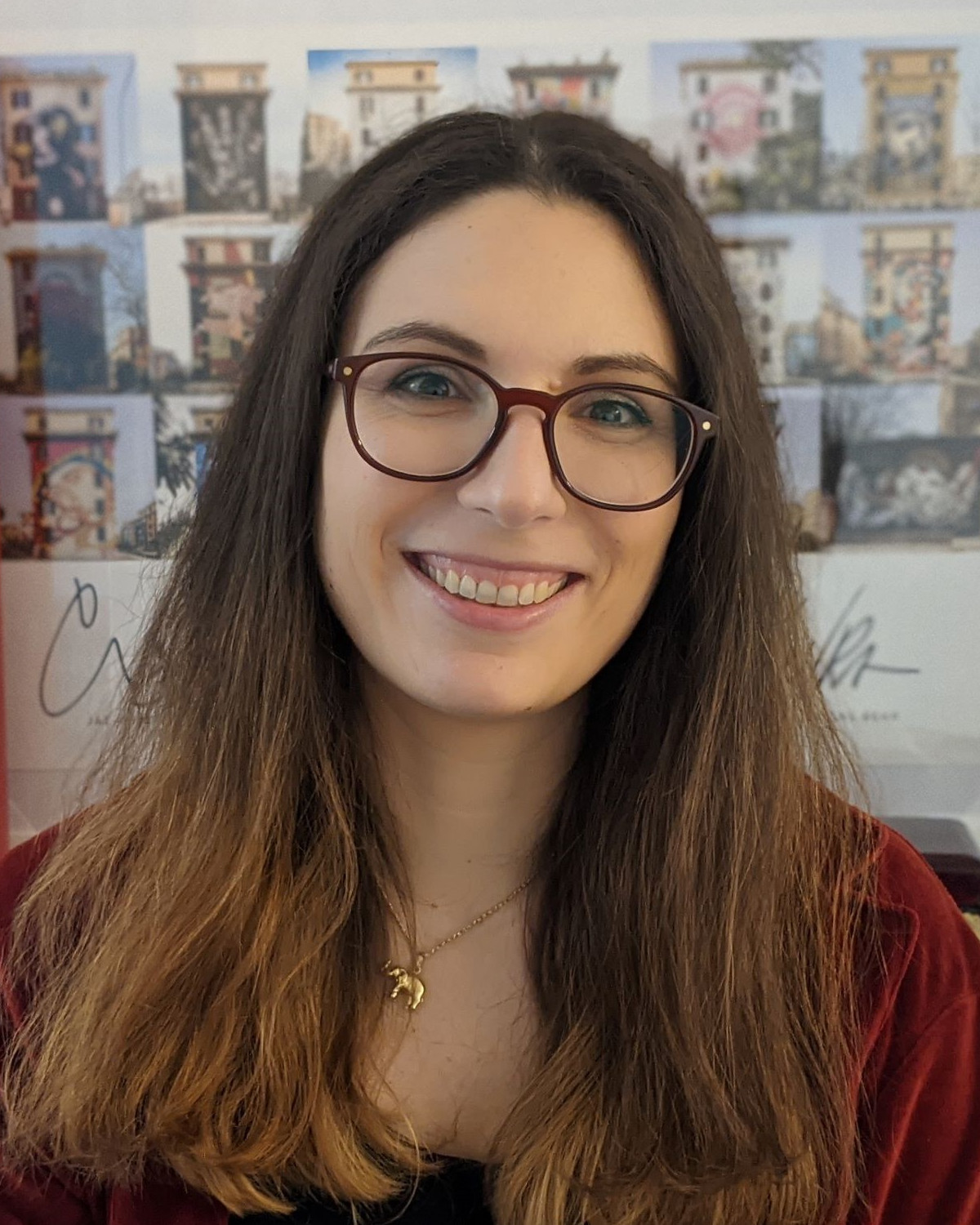}}]{Laura Ferranti} received her Ph.D. degree from Delft University of Technology, Delft, The Netherlands, in 2017. She is currently an assistant professor in the Cognitive Robotics department, Delft University of Technology, Delft, The Netherlands. She is the recipient of an NWO Veni Grant from the Netherlands Organization for Scientific Research (2020) and of the Best Paper Award in Multi-robot Systems at ICRA 2019. Her research interests include optimization and optimal control, model predictive control, reinforcement learning, and embedded optimization-based control with applications in flight control, maritime transportation, robotics, and automotive.
\end{IEEEbiography}

\vfill

\end{document}